\documentclass[11pt]{article}

\usepackage{amsfonts,amsthm,amssymb,amsmath}
\usepackage{graphicx,mathrsfs}
\usepackage[usenames,dvipsnames]{color}
\usepackage[colorlinks]{hyperref}
\hypersetup{
    colorlinks=true,
    linkcolor=red,
    urlcolor=Violet,
    citecolor=blue
    }

\usepackage[top=1in,bottom=1in,left=1in,right=1in,a4paper
]{geometry}
\usepackage{natbib}
\numberwithin{equation}{section}
\numberwithin{table}{section}
\numberwithin{figure}{section}

\usepackage[utf8]{inputenc} 
\usepackage[T1]{fontenc}    
\usepackage{hyperref}       
\usepackage{url}            
\usepackage{booktabs}       
\usepackage{amsfonts}       
\usepackage{nicefrac}       
\usepackage{microtype}      
\usepackage{xcolor}         
\usepackage{float}
\usepackage{graphicx} 
\usepackage{hhline}
\usepackage{multicol}
\usepackage[utf8]{inputenc}
\usepackage[english]{babel}
\usepackage{amssymb}
\usepackage{amsthm}
\usepackage{amsmath}
\usepackage{multirow}
\usepackage{bm}
\usepackage[ruled,linesnumbered]{algorithm2e}
\usepackage{caption}
\usepackage{subcaption}
\usepackage{mathtools}
\usepackage{booktabs}

\newtheorem{assumption}{Assumption}
\newtheorem{definition}{Definition}[section]
\newtheorem{theorem}{Theorem}[section]
\newtheorem{lemma}{Lemma}[section]
\newtheorem{property}{Proposition}[section]

\newtheorem*{notations}{Notations}

\newcommand{\pfactorc}{\sigma_{\varepsilon,\delta}}

\newcommand{\R}{\mathbb{R}}
\newcommand{\E}{\mathbb{E}}

\newcommand{\RB}{c_r}
\newcommand{\CB}{C_B}
\newcommand{\GB}{\zeta}

\DeclareMathOperator*{\argmax}{arg\,max}

\usepackage{pgfplots}
\DeclareUnicodeCharacter{2212}{−}
\usepgfplotslibrary{groupplots,dateplot}
\usetikzlibrary{patterns,shapes.arrows}
\pgfplotsset{compat=newest}

\tikzset{
	cir/.style ={
	ellipse, 
	minimum width =15pt, 
	minimum height =23pt, 
	align=center,
	text width=20pt,
	inner sep=2pt, 
	draw=black, 
	}
}

\tikzset{
	pcir/.style ={
	circle, 
	minimum width =15pt, 
	minimum height =15pt, 
	align=center,
	text width=20pt,
	inner sep=2pt, 
	draw=blue, 
	}
}

\tikzset{
box/.style ={
rectangle, 
rounded corners =3pt, 
minimum width =25pt, 
minimum height =20pt, 
inner sep=2pt, 
draw=black 
}
}

\title{Generalized Linear Bandits with Local Differential Privacy}

\author{%
  Yuxuan Han$^{1}$\footnotemark, Zhipeng Liang$^{2}$\footnotemark[1], Yang Wang$^{1,2}$, and Jiheng Zhang$^{1,2}$ \\
  Department of Mathematics$^{1}$\\
  Department of Industrial Engineering and Decision Analytics$^{2}$\\
  The Hong Kong University of Science and Technology
}

\renewcommand{\thefootnote}{\fnsymbol{footnote}}

\begin{document}

\maketitle
\footnotetext{*Equal contributions.}

\begin{abstract}
	Contextual bandit algorithms are useful in personalized online decision-making. However, many applications such as personalized medicine and online advertising require the utilization of individual-specific information for effective learning, while user's data should remain private from the server due to privacy concerns. This motivates the introduction of local differential privacy (LDP), a stringent notion in privacy, to contextual bandits. In this paper, we design LDP algorithms for stochastic generalized linear bandits to achieve the same regret bound as in non-privacy settings.
  Our main idea is to develop a stochastic gradient-based estimator and update mechanism to ensure LDP.  
  We then exploit the flexibility of stochastic gradient descent (SGD), whose theoretical guarantee for bandit problems is rarely explored, in dealing with generalized linear bandits.
  We also develop an estimator and update mechanism based on Ordinary Least Square (OLS) for linear bandits.
  Finally, we conduct experiments with both simulation and real-world datasets to demonstrate the consistently superb performance of our algorithms under LDP constraints with reasonably small parameters $(\varepsilon, \delta)$ to ensure strong privacy protection. 
\end{abstract}




\renewcommand{\thefootnote}{\arabic{footnote}}

\section{Introduction}
Contextual bandit algorithms have received extensive attention for their efficacy for online decision making in many applications such as recommendation system, clinic trials, and online advertisement \cite{Bietti2018, Slivkins2019, lattimore2020bandit}. 
Despite their success in many applications, intensive utilization of user-specific information, especially in privacy-sensitive domains such as clinical trials and e-commerce promotions, raises concerns about data privacy protection. 
Differential privacy, as a provable protection against identification from attackers \cite{dwork2006calibrating, Dwork2013}, has been put forth as a competitive candidate for a formal definition of privacy and has received considerable attention from both academic research \cite{rubinstein2009learning, dwork2009differential, wasserman2010statistical, smith2011privacy, chaudhuri2011differentially} and industry adoption \cite{erlingsson2014rappor, ding2017collecting, tang2017privacy}. 
While increasing attention has been paid to bandit algorithms with \emph{jointly differential privacy} \cite{Shariff2018, Chen2020}, we introduce in this paper a more stringent notion, \emph{locally differential privacy} (LDP), in which users even distrust the server collecting the data, to contextual bandits.

In contextual bandit, at each time round $t$ with individual-specific context $X _t$, the decision maker can take an action $a_t$ from a finite set (arms) to receive a reward randomly generated from the  distribution depending on the context $X_{t}$ and the chosen arm through its parameter $\theta^{\star}_{a_t}$ which is not unknown to the decision maker. We use the standard notion of expected regret to measure the difference between expected rewards obtained by the action $a_t$ and the best achievable expected reward in this round. While several papers consider the adversarial setting (i.e., $X_t$ can be arbitrary determined in each round), this paper considers the stochastic contextual case where $X_t$ is generated i.i.d.\ from a distribution $P_X$. The goal is to maximize the rewards accumulated over the time horizon. 
An algorithm achieves LDP guarantee if every user involved in this algorithm is guaranteed that anyone else can only access her context (and related information such as the arm chosen and the reward) with limited advantage over a random guess.
Recently there is an emerging steam of works combining LDP and bandit. 
\cite{Basu2019, Ren2020, Chen2020b} consider the LDP contextual-free bandit and design algorithms to achieve the same regret as in the non-privacy setting.   
For contextual bandits, \cite{Zheng2020} considers the adversarial setting. 
Despite their pioneering work, their regret bounds $O(T^{3/4})$ leave a gap from the corresponding non-privacy results $O(T^{1/2})$, which is conjectured to be inevitable. 
A natural question arises: can we close this gap for stochastic contextual bandits? In this paper, we design several algorithms and show that they can achieve the same regret rate in terms of $T$ as in the non-private settings.

If we don't assume any structure on the arms' parameters, the above formulation is referred to as multi-parameter contextual bandits.
If we impose structural assumptions such as all arms share the same parameter (see Section~\ref{sec:ldp-bandits} for details), then the formulation is referred to as single-parameter contextual bandits. 
Although multi-parameter and single-parameter settings can be shown to be equivalent, they need independent analysis and design of algorithms because of their distinct properties based on different modeling assumptions (e.g., \cite{Raghavan2018}).
In this paper, we consider the privacy guarantee in both settings. In fact, multi-parameter setting is more difficult since we need to estimate the parameters for all $K$ arms with sufficient accuracy to make good decisions. 
However, privacy protection also requires protecting the information about which arm is pulled in each round.
Such a requirement hinders the identification of optimal arm and may incur considerable regret in the decision process. A proper balance between privacy protection and estimation accuracy is the key to design algorithms with desired performance guarantee in this setting.

\begin{table}[!htbp]
	\centering
	\begin{tabular}{lllll}
	\toprule
	Result & Regret & Context & Parameter & $\beta$-Margin\\
	\midrule
	\cite{Zheng2020}
	& $\tilde{O}(T^{3/4}/\varepsilon)$ & Adversary & Both & No Margin\\
	\cline{1-5}
		Theorem~\ref{ub-single-worst} & $\tilde{O}(T^{1/2}/\varepsilon) $   & Stochastic & Single & No Margin\\
	Theorem~\ref{thm:opt-single-margin} & $O(\log T/\varepsilon^2)$   & Stochastic & Single & $\beta=1$\\
	Theorem~\ref{thm:opt-single-margin} & $\tilde{O}(T^{\frac{1-\beta}{2}}/\varepsilon^{1+\beta})$   & Stochastic & Single & $0\le \beta<1$\\
	Theorem~\ref{mainthmsm} & $O((\log T/\varepsilon)^2)$ & Stochastic & Multiple &  $\beta=1$ \\
	Theorem~\ref{mainthmsm} & $\tilde{O}(T^{\frac{1-\beta}{2}}/\varepsilon^{1+\beta})$ & Stochastic & Multiple &  $0<\beta<1$ \\
	\bottomrule
	\end{tabular}
	\caption{Summary of our main results in $(\varepsilon, \delta)$-LDP, where $\tilde{O}(\cdot)$ omits poly-logarithmic factors.}
	\label{tab:contribution}
	\end{table}

\textbf{Contributions.}
We organize our results for various settings in Table~\ref{tab:contribution}. Our main contributions can be summarized as follows:

1. We develop a framework for implementing LDP algorithms by integrating greedy algorithms with a private OLS estimator for linear bandits and a private SGD estimator for generalized linear bandits. We prove that our algorithms achieve regret bound matching the corresponding non-privacy results. 

2. In the multi-parameter setting, to ensure the privacy of the arm pulled in each round, we design a novel LDP strategy by simultaneously updating all the arms with synthetic information instead of releasing the pulled arm. 
By conducting such synthetic updates for unselected arms, we protect the information of the pulled arm from being identified by the server or other users. 
This is at the cost of corrupting the estimation of the un-selected arms. 
To deal with this issue, we design an elimination method that is only based on data collected during a short warm up period.
We show that such a mechanism can be combined with the OLS and SGD estimators to achieve the desired performance guarantees.

3. We introduce the SGD estimator to bandit algorithms to tackle generalized linear reward structure.
To the best of our knowledge, few papers have ever considered SGD-based bandit algorithms. 
Theoretical regret bounds are established in \cite{ding2021efficient} by combining SGD and Thompson Sampling, while most of the others are limited to empirical studies \cite{Bietti2018, riquelme2018deep}.
We establish such theoretical regret bounds for SGD-based bandit algorithms. 
Our private SGD estimator for bandits is highly computationally efficient, and more importantly, greatly simplifies the data processing mechanism for LDP guarantee.

\section{Preliminaries}
\label{sec:prelim}
\begin{notations}
	We start by fixing some notations that will be used throughout this paper. 
	For a positive integer $n$, $[n]$ denotes the set $\{1,\cdots, n\}$. $\lvert A \rvert$ denotes the cardinality of the set $A$.
	$\lVert \cdot \rVert_2$ is Euclidean norm.
	$W(i,j)$ denotes the element in the $i$-th row and $j$-th column of matrix $W$. 
	We write $W>0$ if the matrix $W$ is symmetric and positive definite. We denote $I_d$ as the $d$-dimensional identity matrix.
	Let $\otimes$ denote the Kronecker product.
	Let $B_r^d$ denote the $d$-dimensional ball with radius $r$ and $S^{d-1}_r$ denotes the $(d-1)$-dimensional sphere for the ball. 
	Given a set $A$, Unif$(A)$ denote the uniform distribution over $A$.
	For a tuple $(Z_{i,j})_{i\le N, j\le M}$ and $1\le k_1< k_2\le M$, we denote $Z_{i, k_1:k_2} = (Z_{i,k_1}, \cdots, Z_{i,k_2})$. 
	We adopt the standard asymptotic notations: for two non-negative sequences $\{a_n\}$ and $\{b_n\}$, $\{a_n\}=O(\{b_n\})$ iff $\lim \sup_{n\rightarrow \infty}a_n/b_n<\infty$, $a_n = \Omega (b_n)$ iff $b_n = O(a_n)$, $a_n = \Theta(b_n)$ iff $a_n = O(b_n)$ and $b_n = O(a_n)$. We also write $\tilde{O}(\cdot)$, $\tilde{\Omega}(\cdot)$ and $\tilde{\Theta}(\cdot)$ to denote the respective meanings within multiplicative logarithmic factors in $n$. 
\end{notations}

\subsection{Local Differential Privacy}
\label{sec:ldp}    

\begin{definition}[Local differential privacy] 
	\label{def:ldp}
	We say a (randomized) mechanism $M:\mathcal X\to \mathcal Z$ is $(\varepsilon,\delta)$-LDP, if for every $x\neq x'\in \mathcal{ X}$ and any measurable set $C\subset \mathcal Z $ we have 
	\begin{align*}
    P(M(x)\in C) \leq e^\varepsilon P(M(x')\in C)+\delta.
  \end{align*}
	When $\delta=0$, we simply denote $\varepsilon$-LDP.
\end{definition}
We now present some tools that will be useful for our analysis. 

\begin{lemma}[Gaussian Mechanism \cite{dwork2006our,Dwork2013}]
	\label{lem:gaussian}
	For any $f: \mathcal{X} \rightarrow$ $\R^{n}$, let $\sigma_{\varepsilon,\delta} = \frac{1}{\varepsilon} \sup_{x,x^{\prime}\in \mathcal{X}} \lVert f(x) -f(x^{\prime})\rVert_2\sqrt{2 \ln (1.25 / \delta)}$.
	The Gaussian mechanism, which adds random noise independently drawn from distribution $\mathcal{N}(0, \sigma^{2}_{\varepsilon,\delta}I_n)$ to each output of $f$, ensures $(\varepsilon, \delta)$-LDP.
\end{lemma}

Although all our results can be extended in parallel to $\varepsilon$-LDP if using Laplacian noise instead of Gaussian noise, we focus on $(\varepsilon, \delta)$-LDP in this paper.
Besides the Gaussian mechanism, we also use the following privacy mechanism for bounded vectors.

\begin{lemma}[Privacy Mechanism for $l_2$-ball \cite{duchi2018minimax}]
	\label{lem:l2}
	For any $R>0$, let $r_{\varepsilon,d} = R \dfrac{\sqrt{\pi}}{2}\dfrac{e^\varepsilon+1}{e^\varepsilon-1}\dfrac{d\Gamma(\frac{d+1}{2})}{\Gamma(\frac{d}{2}+1)}$ where  and $\Gamma$ is the Gamma function.
	For any  $x\in B^d_R$, consider the mechanism $\Psi_{\varepsilon,R}:B^d_R\to S^{d-1}_{r_{\varepsilon,d}}$ of generating $Z_x$ as the follows. 
	First, generate a random vector $\tilde{X} = (2b-1)x$ where $b$ is a Bernoulli random variable with success probability $\frac{1}{2}+\frac{\lVert x\rVert_2}{2R}$. Next, generate  random vector $Z_x$ via 
	\begin{align*}
		Z_x\sim \left\{\begin{matrix}
			\text{Unif}\{z \in \R^d:  z^T \tilde{X}>0,\lVert z\rVert_2 = r_{\varepsilon,d}\} \textrm{ with probability }e^\varepsilon/(1+e^\varepsilon)\\
			\text{Unif}\{z \in \R^d:  z^T\tilde{X}\leq 0,\lVert z\rVert_2 = r_{\varepsilon,d}\} \textrm{ with probability } 1/(1+e^\varepsilon).
		\end{matrix}\right.
	\end{align*}
 	Then $\Psi_{\varepsilon,R}$ is $\varepsilon$-LDP and $\E[\Psi_{\varepsilon,R}(x)] = x.$
\label{privacy-bounded-vector}
\end{lemma}
\begin{lemma}[Post-Processing property \cite{Dwork2013}]
	\label{lem:post}
	If  $M:\mathcal X\to\mathcal Y$ is $(\varepsilon,\delta)$-LDP and $f:\mathcal Y\to \mathcal Z$ is a fixed map, then $f\circ M:\mathcal X\to\mathcal Z$ is $(\varepsilon,\delta)$-LDP.  \end{lemma}

\begin{lemma}[Composition property \cite{Dwork2013}] If $M_1:\mathcal X\to\mathcal Z_1$ is  $(\varepsilon_1,\delta_1)$-LDP and $M_2:\mathcal X\to\mathcal Z_2$ is $(\varepsilon_2,\delta_2)$-LDP, then $M = (M_1,M_2): \mathcal X\to\mathcal Z_1\times \mathcal Z_2$ is $(\varepsilon_1+\varepsilon_2,\delta_1+\delta_2)$-LDP.
	\label{lem:composition} 
\end{lemma}

\subsection{Local Differential Privacy in Bandit}
\label{sec:ldp-bandits}
We consider contextual bandits with LDP guarantee in the context of the user-server communication protocol described in Figure~\ref{fig:protocol}.
The user in round $t$ with context $X_t\in\R^d$ receives (processed) historical information $S_{t-1}$ from the server, and chooses an action $a_t\in [K]$ to obtain a random reward $r_t = v(X_t, a_t) + \epsilon_t$ . 
Define $\mathcal{F}_t$ as the filtration of all historical information up to time $t$, i.e., $\mathcal{F}_t = \sigma(X_1, \cdots, X_t, \epsilon_1,\cdots, \epsilon_{t-1})$, and we require $\E[\epsilon_t\lvert \mathcal{F}_t ] = 0,\E[\exp( \lambda\epsilon_t)\lvert \mathcal{F}_t ]\leq\exp(\dfrac{\sigma_\epsilon^2\lambda^2}{2}),\forall\lambda\in\R $. 
Then the user processes the tuple $(X_t, r_t)$ by some mechanism $\varphi$ with LDP guarantee and send the processed information $Z_t = \varphi(X_t, r_t)$ to the server. 
After receiving $Z_t$, the server updates the historical information $S_t$ to get $S_{t+1}$. 
We consider the generalized linear bandits by allowing $v(X_t, a_t) = \mu(X_t^T\theta^{\star}_{a_t})$, where $\mu:\R\to \R$ is a link function and $\theta^{\star}_i\in\R^d$ is the underlying parameter of the $i$-th arm. For a fix time t, we denote $a^{*}_t = \argmax_{i\in [K]} \mu(X_{t}^T \theta^{\star}_i)$.
The regret over time horizon $T$ is $\text{Reg}(T) = \sum_{t=1}^T \left( \mu(X_{t}^T \theta^{\star}_{a_t^{*}}) - \mu(X_{t}^T \theta^{\star}_{a_t})\right)$. 
If we don't assume any structure on $\{\theta^{\star}_i\}_{i\in [K]}$, we refer it as the multi-parameter setting. 
We also consider $d$-dimensional single-param setting by assuming $\theta_{i}^{\star} = e_i \otimes \theta^{\star}$ for some $\theta^{\star}\in\R^{d}$ where $\{e_i\}_{i\in[K]}$ is canonical basis of $\R^K$.
In this case, $x_{t,i}\in\R^d$ is the $i$-th segment of $X_t\in\R^{dK}$ and $X_{t}^T\theta_{i}^{\star} =x_{t,i}^T \theta^{\star}$, so choosing arm~$i$ becomes choosing the $i$-th segment $x_{t,i}$ of the context.

\begin{figure}[!htbp]
\centering
\resizebox{\linewidth}{!}{
\begin{tikzpicture}
	\usetikzlibrary{shapes}
	\node at(-1,2.2) {User Side};
	\node at(-1,0.5) {Server Side};
	\node at (0.5, 0.5) {$\cdots$};
	\node at (15, 0.5) {$\cdots$};
	\node at (0.5, 2.2) {$\cdots$};
	\node at (15, 2.2) {$\cdots$};
	\node[cir] (1) at (2, 1.8) {$X_t$};
	\node[cir] (2) at (2, 3) {$r_t$};
	\node[pcir] (3) at (5, 2.5) {$Z_t$};
	\node[cir] (4) at (8, 1.8) {$X_{t+1}$};
	\node[cir] (5) at (8, 3) {$r_{t+1}$};
	\node[pcir] (6) at (11, 2.5) {$Z_{t+1}$};
	\node[box] (7) at (2, 0.5) {$S_t$};
	\node[box] (8) at (8, 0.5) {$S_{t+1}$};
	\node[box] (9) at (14, 0.5) {$S_{t+2}$};
	\draw[-latex, line width=0.30mm] (1) -- node[left] {$a_t$} ++ (2);
	\draw[-latex, dotted, blue, line width=0.30mm] (1) -- (3);
	\draw[-latex, dotted, blue, line width=0.30mm] (2) -- (3);
	\draw[-latex, line width=0.30mm] (4) -- node[left] {$a_{t+1}$} ++ (5);
	\draw[-latex, dotted, blue, line width=0.30mm] (4) -- (6);
	\draw[-latex, dotted, blue, line width=0.30mm] (5) -- (6);
	\draw[-latex, line width=0.30mm] (7) -- (1);
	\draw[-latex, line width=0.30mm] (3) -- (8);
	\draw[-latex, line width=0.30mm] (8) -- (4);
	\draw[-latex, line width=0.30mm] (6) -- (9);
	\draw[-latex, line width=0.30mm] (7) -- (8);
	\draw[-latex, line width=0.30mm] (8) -- (9);
	\draw[rounded corners, dotted] (-3,0) rectangle (16,1.1);
	\draw[rounded corners, dotted] (-3,1.2) rectangle (16,3.8);
\end{tikzpicture}
}
\caption{User-server communication protocol}
\label{fig:protocol}
\end{figure}
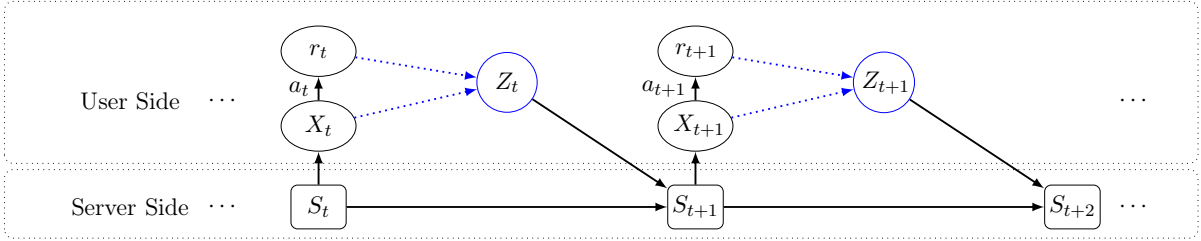

In the rest of paper, we always assume that $\lVert \theta_i^{\star}\rVert_2\leq 1,\forall i\in [K]$, the reward is bounded by $\RB$ and the context is bounded by $\CB$, our analysis can be easily generalized to the case where $\epsilon_t$ and the context follow sub-gaussian distributions.
We also impose regularize assumptions on the link function, which are common in previous work \cite{Zheng2020, ren2020batched, toulis2014statistical} and the corresponding family contains a lot of commonly-use model, e.g., linear model, logistic model.

\begin{assumption} 
	\label{assump:link}
	The link function $\mu$ is continuously differentiable, Lipschitz and there exists some $\GB>0$ such that $\inf_{x\in [-\CB,\CB]}\mu'(x)=\GB>0$.
\end{assumption}

\section{Single-Parameter Setting}
\label{sec:single}

In this section, we develop a LDP contextual bandit framework (Algorithm~\ref{LDP-Single}) by combining statistical estimation and privacy mechanisms in the single-param bandit setting to achieve optimal regret bound in various cases.
We use an abstract privacy mechanism $\psi$ in \eqref{eq:ldp-estimator} and estimator $\varphi$ in \eqref{eq:mechanism} to allow the plug-in of various components.

\begin{algorithm}
	\KwIn{Time horizon $T$; Privacy Level $\varepsilon,\delta$.}
	\textbf{Initialization:} Setting $ \hat{\theta}_{0} = \bm{0}.$ \\
	\For{ $t\leftarrow 1$ to  $T$}
	{
		\textbf{User side:}\\
			\quad Receive  $\hat{\theta}_{t-1}$ from the server.\\
			\quad Pull arm $a_t =  \text{argmax}_{a\in [K]} x_{t,a}^T \hat{\theta}_{t-1}$ and receive $r_{t}.$	\\
			\quad Generate $Z_{t}$ by 
			\begin{equation}
				\label{eq:ldp-estimator}
				Z_t=\psi_t(x_{t,a_t}, r_t; \hat{\theta}_{t-1}).
			\end{equation}\\
		\noindent\textbf{Server side:} \\
			\quad Receive $Z_{t}$ from the user.  \\
			\quad Update the estimation via 
			\begin{equation}
				\label{eq:mechanism}
				\hat{\theta}_t=\varphi_t(Z_1,\dots,Z_{t};\hat{\theta}_{t-1}).	
			\end{equation}
	}
	\caption{LDP Single-parameter Contextual Bandit}
	\label{LDP-Single}
\end{algorithm}

\subsection{Privacy Guarantee}
\label{sec:single-algo}

For the linear case where the link function $\mu(x) = x$, we can use the following ordinary least square (OLS) estimator. 
Let with $\pfactorc = 2 \sqrt{2 \ln (1.25/ \delta)}/\varepsilon$.
Define $M_t=x_{t,a_t}x_{t, a_t}^T+W_t$ where $W_t$ is a random matrix with $W_t(i,j)\sim \mathcal{N}(0, 4\CB^2 \pfactorc^2)$ and $W_t(j,i) = W_t(i,j)$, and $u_t=r_tx_{t, a_t}+\xi_t$ where $\xi_t$ is a random vector following distribution $\mathcal{N}(0, \CB^2 \RB^2\pfactorc^2 I_d)$.
The OLS privacy mechanism and the corresponding estimator are
\begin{align}
	\label{eq:OLS-mechanism}
	\psi^{OLS}_t(x_{t,a_t},r_t; \hat{\theta}_{t-1}) &= (M_t,u_t),\\
	\label{eq:OSL-estimator}
	\varphi^{OLS}_t(Z_1,\dots,Z_t; \hat{\theta}_{t-1}) &= 
		\big(
			\sum_{i=1}^t M_i +\tilde{c} \sqrt{t} I
		\big)^{-1}
		\sum_{i=1}^t u_i,
\end{align}
where $\tilde{c}>0$ is to be determined. We have the following LDP guarantee using the Gaussian mechanism (Lemma~\ref{lem:gaussian}) and post-processing (Lemma~\ref{lem:post}).
\begin{property}\label{privacy:single-OLS}
	Algorithm~\ref{LDP-Single} with the private OLS update mechanism $\psi^{OLS}_{t}$ and estimator $\varphi_{t}^{OLS}$ is $(\varepsilon, \delta)$-LDP.
\end{property}
For the general link function $\mu$, its non-linearity adds to the difficulty in terms of both privacy-preserving and bandits. 
To estimate parameters in generalized linear bandits, one common approach to use a maximum likelihood estimator (MLE)  at each step. 
In contrast to OLS solution, MLE does not have a close form solution with simple sufficient statistics in general. 
Thus, solving an MLE optimization procedure requires using all the previous data points and conducting costly operations at each round, resulting in time complexity and memory usage increasing with time. 
Instead, we use a one-step stochastic gradient approximation to incrementally update the estimator with the new observation at each round. 
To obtain a LDP version of this approximation, we use the LDP $l_2$-ball mechanism in Lemma~\ref{privacy-bounded-vector}.
\begin{align}
	\label{eq:SGD-mechanism}
	\psi_t^{SGD}(x_{t,a_t},r_t; \hat{\theta}_{t-1}) &= \Psi_{\varepsilon,R}
		\left( 
			\big(
				\mu( x_{t, a_t}^T\hat{\theta}_{t-1})-r_t
			\big)
			x_{t, a_t}
		\right), \\
		\label{eq:SGD-estimator}
	\varphi_t^{SGD}(Z_1,\dots,Z_t; \hat{\theta}_{t-1}) &= \hat{\theta}_{t-1}-\eta_t \psi^{SGD}_t.
\end{align}
where $\eta_t>0$ is the stepsize to be determined and $R = 2\RB\CB $. Similarly, we can prove the following LDP guarrantee using the $l_2$-ball mechanism Lemma~\ref{lem:l2} and post-processing Lemma~\ref{lem:post}.

\begin{property}\label{privacy:single-SGD}
	Algorithm~\ref{LDP-Single} with the private SGD update mechanism $\psi^{SGD}_{t}$ and estimator $\varphi_{t}^{SGD}$ is $\varepsilon$-LDP.
\end{property}

\subsection{Regret Analysis}
\label{sec:single-regret}
To derive the regret bound of our framework, we need the following assumptions on the marginal distribution $P_X$ of the stochastic contexts $\{x_{t,a}\}_{a\in [K]}$.

\begin{assumption}
	\label{singleeigenvalue} 
	There exists some $\kappa_u>0$ such that $\lambda_{\max}(\Sigma_a)\leq \frac{\kappa_u}{d}$ where $\Sigma_a$ is the covariance matrix of $P_X$ and $\lambda_{\max}(\Sigma_a)$ is the largest eigenvalues of $\Sigma_a$.
\end{assumption}

\begin{assumption} 
	\label{as:pstar}
	 For every $\lVert u\rVert_2 = 1$, denote $a^* = \argmax_{a\in [K]} x_{t,a}^Tu$, there exist some $\kappa_l>0, p_*>0$ such that $P_{u}((x^T v)^2 > \kappa_l/d) \geq p_*$ holds for any $u\in S^{d-1}_1$, where $P_u(\cdot)$ is the distribution of $x_{t,a^{*}}$. 
\end{assumption}

Similar assumptions are common in the analysis of single-parameter contextual bandits, e.g. \cite{ding2021efficient,Han2020}, and our conditions contain a wide range of distributions, including sub-gaussian with bounded density. See appendix~\ref{Randomness} for discussion. Now we can show that our framework indeed achieves optimal regret bound.

\begin{theorem}
	\label{ub-single-worst}
	Under Assumptions~\ref{singleeigenvalue} and \ref{as:pstar}, with the choice of $\tilde{c} = 2\pfactorc (4\sqrt{d} + 2\log (2T/\alpha))$ in \eqref{eq:OSL-estimator}, Algorithm \ref{LDP-Single} with OLS mechanism $\psi^{OLS}_{t}$ and estimator $\varphi_{t}^{OLS}$ achieve the following regret with probability at least $1-\alpha$ for some constant $C$,
	\begin{align*}
		\text{Reg}(T)
		& \le C \sqrt{T} (\CB (\pfactorc +\sigma_{\epsilon}) d \frac{\sqrt{(d+\log(T/\alpha))\log (KT/\alpha)}}{\kappa_lp_{*}} + o(1))
	\end{align*}

	Under Assumptions~\ref{assump:link}--\ref{as:pstar}, with the choice of $\eta_t = c' d/(\kappa_l \zeta p_{*}t)$ for some $c'>1$ in \eqref{eq:SGD-estimator}, Algorithm~\ref{LDP-Single} with SGD mechanism $\psi^{SGD}_{t}$ and estimator $\varphi_{t}^{SGD}$ achieves the following regret with probability at least $1-\alpha$ for some constant $C$, \begin{align*}
		\text{Reg}(T)
		\le C \sqrt{T}(\frac{r_{\varepsilon, d}\sqrt{d}}{\zeta \kappa_l p_{*}} \log\log(T/\alpha) + o(1)).
	\end{align*}
	with $o(1)$ means some factor that turns to $0$ as $T\to\infty$.
\end{theorem}
In the algorithm we shift the sample covariance matrix by $\tilde{c}\sqrt{t}$ to ensure the positive-definiteness of the noise matrix as in \cite{Shariff2018}. Such a shift guarantee the estimation accuracy in the early stage. 
Note that the optimal worst-case regret bound in the non-privacy case is $\tilde{O}(T^{1/2})$, our results show that we can achieve the same regret bound as in the non-privacy case in terms of time $T$. In fact, we can show a $\Omega(\sqrt{T}/\varepsilon)$ lower bound in this setting even when $K=2$, which verified our optimal dependence on both $T$ and $\varepsilon$. 
\begin{theorem}\label{thm:single-lower minimax bound} For $\theta\in\R^d$ and an algorithm $\pi$, we denote $\E[\text{Reg}_{\pi}(T;\theta)]$ the expectation regret of $\pi$ when the underlying parameter is $\theta$. When $K = 2$ and $x_{t,a}\sim \mathcal{N}(0,I_d/d)$ are independent over $a\in [K]$, we have for any possible $\varepsilon$-LDP algorithm \textcolor{black}{$\pi$}, $
		  \sup_{\theta^{\star}:\lVert \theta^{\star}\rVert_2\leq 1 }\E[\text{Reg}_{\textcolor{black}{\pi}}(T;\theta^{\star})] = \Omega(\sqrt{T}/\varepsilon).
		$
\end{theorem}
Given the best known $O(T^{3/4})$ regret bound of  adversarial contextual LDP bandit in \cite{Zheng2020}, our $O(\sqrt{T}/\varepsilon)$ result points out a possible gap between stochastic contextual bandits and adversarial contextual bandits under the LDP constraint. 
The bounds given above are problem-independent, which do not dependent on the underlying parameters. If we consider an additional assumption that there is a gap between the optimal arm and the rest, which is usually the case when the number of contexts is small, then we can obtain sharper bounds than the problem-independent ones in Theorems~\ref{ub-single-worst}.  
\begin{assumption}[$(\gamma, \beta)$-margin condition]
	\label{margin-condition-single}
	We say $P_X$ satisfies the $(\gamma,\beta)$-strong margin condition with $\gamma>0,0<\beta\leq 1$, if for $\triangle_t: = \mu(x_{t,a^*_{t}}^T\theta^{\star})-\max_{j\neq a_t^*}\mu(x_{t,j}^T\theta^{\star})$ and $h\in [0,b]$ with some positive constant $b$,  we have $ \mathbb{P} [\triangle_t\leq h]\leq \gamma h^\beta$.
\end{assumption}

\begin{theorem}
	\label{thm:opt-single-margin}
	Under Assumptions~\ref{singleeigenvalue}--\ref{margin-condition-single} with the same choice of $\tilde{c}$ in Theorems~\ref{ub-single-worst}, 
	Algorithm \ref{LDP-Single} with OLS mechanism $\psi^{OLS}_{t}$ and estimator $\varphi_{t}^{OLS}$ achieves the following regret with probability at least $1-\alpha$ for some constant $C$,
\begin{align*}
	\text{Reg}(T)\leq C\cdot \left\{
		\begin{array}{ll}
		{\gamma \CB}\log T[ (\dfrac{\CB d ( \CB \sigma_\epsilon+\pfactorc )\sqrt{d+\log(T/\alpha)} }{\kappa_l p_*})^2+o_{\beta,\gamma}(1)], &\beta = 1,\\
		\dfrac{\gamma \CB}{1-\beta}T^{\frac{1-\beta}{2}}[ ( \dfrac{\CB d(\CB \sigma_\epsilon+\pfactorc )\sqrt{d+\log(T/\alpha)} }{\kappa_l p_*})^{1+\beta}  +o_{\beta,\gamma}(1)], &0\le \beta <1.
	\end{array}
	\right.\end{align*}

	Under Assumptions~\ref{assump:link}--\ref{margin-condition-single} and  with the same choice of $\eta_t$ in Theorems~\ref{ub-single-worst}, 
	Algorithm \ref{LDP-Single} with SGD mechanism $\psi^{SGD}_{t}$ and estimator $\varphi_{t}^{SGD}$ achieves the following regret with probability at least $1-\alpha$ for some constant $C$,
	\begin{align*}
	\text{Reg}(T)\leq C\cdot 
	\left\{
		\begin{array}{ll}
			{\gamma L\CB} \log T [(\dfrac{r_{\varepsilon,d} L d\CB\sqrt{\log(\log(T)/\alpha)} }{\zeta\kappa_l p_*})^2+o_{\beta,\gamma}(1)], &\beta = 1,\\
			\dfrac{\gamma L\CB}{1-\beta}T^{\frac{1-\beta}{2}} [(\dfrac{r_{\varepsilon,d}Ld\CB\sqrt{\log(\log (T)/\alpha)} }{\zeta\kappa_l p_*})^{1+\beta} +o_{\beta,\gamma}(1)], &0\le \beta <1.
		\end{array}
	\right. 
\end{align*}
with $o_{ \beta, \gamma}(1)$ being a factor depending on $\beta, \gamma$ that converges to $0$ as $T\to\infty$.
\end{theorem}

\section{Multi-parameter Setting}
\label{multisection}

In this section, we present our LDP framework for the multiple parameter setting. Compared with the single parameter setting, this framework introduces three non-trivial components to match classical regret bounds while still guarantee LDP: warm up, synthetic update and elimination.

\begin{algorithm}
		\KwIn{Time horizon $T$; Warm up period length $s_0$; Privacy Level $\varepsilon,\delta$.}
		\textbf{Initialization:} Setting $ \hat{\theta}_{0,i} = 0,i\in [K].$ \\
		\For{ $t\leftarrow 1$ to  $Ks_0$}{
			\textbf{User side:}\\
				\quad Receiving $\hat{\theta}_{t-1, 1:K}$ from the server. \\
				\quad Pulling arm $a_t\coloneqq (t\text{ mod }K)+1 $ and receive $r_t$. \\
				\quad Generate and update $Z_{t,i} =\bm 1\{a_t = i\} \psi_t(X_t,r_t;\hat{\theta}_{t-1,i}),i\in [K]$ to the server. \\
			\textbf{Server side:} \\
			  \quad Receive the update $Z_{t,1:K}$ from the user.   \\
				\quad Re-estimate parameters via $\hat{\theta}_{t,i}\coloneqq\varphi_{t}(Z_{1,i},\dots,Z_{t,i}),\forall i\in [K]. $
		}
		\For{ $t\leftarrow Ks_0+1$ to  $T$}{
			\textbf{User side:}\\
				\quad Receive  $\hat{\theta}_{t-1, 1:K}$ from the server.\\
				\quad Determine a subset $\hat{K}_t$ of $[K]$ by setting 
				\begin{equation}
					\hat{K}_t\coloneqq \{a\in [K]:  X_t^T\hat{\theta}_{Ks_0,a}> \max_{a\in [K]}  X_t^T\hat{\theta}_{Ks_0, a}-\dfrac{h}{2} \}
					\label{Elimination}
				\end{equation}\\
				\quad Pulling arm $a_t\coloneqq \text{argmax}_{a\in \hat{K}_t}  \mu(X_t^T\hat{\theta}_{t-1, a})$ and receive $r_{t}.$	\\
				\quad Generating information for all arms $\{Z_{i,t}\}_{i\in [K]}$ by setting 
				\begin{align*}
					Z_{i,t} = \left\{
						\begin{array}{ll}
							\psi_t(X_t,r_t; \hat{\theta}_{t-1,i}) &\text{ if } a_t = i,\\
							\psi_t(\bm 0,0; \hat{\theta}_{t-1, i}) &\text{ otherwise.}
						\end{array}
						\right.
				\end{align*}\\
			\noindent\textbf{Server side:} \\
			  \quad Receive the update $\{Z_{i,t}\}_{i\in [K]}$  from the user.  \\
				\quad Re-estimate parameters via $$\hat{\theta}_{t,i}\coloneqq\varphi_{t}(Z_{1,i},\dots,Z_{t,i} ).$$
		}
		\caption{LDP Multi-parameter Contextual Bandit}
		\label{algo:LDP-Multi}
\end{algorithm}

\textbf{Warm up.} 
In the warm up stage, all arms are given equal opportunities to be explored for a preliminary estimation of their parameters. 
Such estimation does not aim for the accuracy to select the optimal arm with high probability. Instead, we only need accuracy at the level of ruling out the substantially inferior arms.
Thus, this stage only needs $O(\log T)$ steps. 

Since the actions in this stage are independent of the contexts, there is no need to protect the pulled arm. 
However, we still need to protect the contexts by using a privacy mechanism similar in the single-parameter setting.
 
\textbf{Synthetic update.}
After the warm up, we need to make decisions based on the contexts to achieve vanishing regret.
In order to obtain the privacy guarantee, we introduce our synthetic update mechanism. 
Although in each time only one arm is pulled, we create synthetic data for all unselected arms.
In this way, the server receives synthetic feedback about all arms, regardless of whether it is selected or not, and thus cannot figure out which one is selected.

Another method to provide LDP protection for the selected arm is to ensure the action $a_t$ satisfies LDP. However, the regret will grow linearly, as shown in \cite{Shariff2018}.

\textbf{Elimination.} We use the information obtained during warm up to exclude obviously inferior arms.
Such a method has been applied in \cite{Bastani2017} to guarantee a certain kind of independence of the information in each round.
However, we use this method for a different purpose. The necessity of such an elimination strategy comes from protecting privacy in the multi-parameter setting.
Although we have obtained an estimation to a certain level of accuracy in the warm up stage, our knowledge on un-selected arms will be gradually corrupted by the noise incurred in the synthetic update in each round.
Such corruption will make us fail to distinguish arms that are possibly optimal from the surely sub-optimal ones.
To avoid corruption, we may need to pick the sub-optimal arms frequently but this will result in large regret.
That is why we use the warm up information to eliminate the arms with extremely poor performance as in \eqref{Elimination}.

\subsection{Privacy Guarantee}
\label{sec:multi-privacy}
The OLS/SGD mechanisms and estimators are the same as \eqref{eq:OLS-mechanism}--\eqref{eq:SGD-estimator} in the single-parameter setting. To prevent the server from distinguishing the selected arm from the other $K-1$ arms, a straightforward idea is to use $(\varepsilon/K,\delta/K)$-LDP mechanism for the synthetic update by composition property in lemma~\ref{lem:composition}. However, we can prove that our algorithm can still achieve the same LDP guarantee with a much less stringent privacy mechanism, say $(\varepsilon/2,\delta/2)$-LDP, in Propositions ~\ref{privacy:multi-ols} and \ref{privacy:multi-sgd}. 
\begin{property}
	\label{privacy:multi-ols}
	Algorithm~\ref{algo:LDP-Multi} with the private OLS update mechanism $\psi^{OLS}_{t}$ and estimator $\varphi_{t}^{OLS}$ is $(\varepsilon, \delta)$-LDP.
\end{property}

\begin{property}
	\label{privacy:multi-sgd}
	Algorithm~\ref{algo:LDP-Multi} with the private SGD update mechanism $\psi^{SGD}_{t}$ and estimator $\varphi_{t}^{SGD}$ is $\varepsilon$-LDP.
\end{property}

\subsection{Regret Analysis}
\label{sec:multi-regret}

\begin{assumption}[Diversity condition]
	\label{diversecondition} 
	Let $K_{opt}$ and $K_{sub}$ be a partition of $[K]$ such that for any $i\in K_{sub}$, $\mu( X^T\theta_i) <\max_{j\neq i} \mu(X^T\theta_j)-h_{\text{sub}}$ for some $h_{\text{sub}}>0$ and every $X\in\mathcal X$. 
	For any $i\in K_{opt}$ define the set $U_i\coloneqq \{X: \mu( X^T\theta_i)>\max_{j\neq i}\mu(X^T\theta_j)\}$. 
	There exists $\kappa_l>0,p'>0$ such that for all $i\in K_{opt}$ and unit vector $v$,$\mathbb{P}((v^TX)^2\bm 1\{X\in U_i \} \geq \kappa_l/K_{opt})>{p'}$.
\end{assumption}

\begin{assumption}[$(\gamma, \beta)$-margin condition] 
	This is almost identical to Assumption~\ref{margin-condition-single} except that we replace $\triangle_t$ with $\triangle_t\coloneqq \mu(X_{t}^T\theta_{a^*_t})-\max_{j\neq a_t^*}\mu(X_{t}^T{\theta_{j}})$.
	\label{betamargin}
\end{assumption}

In our algorithm, diversity condition guarantees that conditioning on the arm $i$ is pulled, the distribution of $X_t$ still can provide enough information about $\theta_i$. We would remark here that we need no longer any deterministic gap in the definition of $U_i$, which weakens the assumption made in \cite{Bastani2020},\cite{Bastani2017}.
Now we are in the suited position to present our theoretical guarantee of the algorithm.

\begin{theorem}
	\label{mainthmsm}
	Under Assumptions~\ref{assump:link}, \ref{diversecondition} and \ref{betamargin},  with the choice of $\tilde{c}  = 2\sigma_{\varepsilon/2,\delta/2}(4\sqrt{d}+2\log(2TK/\alpha))$ in \eqref{eq:OSL-estimator}, $s_0 = C\cdot K(\dfrac{\CB\sigma_\epsilon+\sigma_{\varepsilon,\delta}}{\min\{\lambda_0,h\} p'\kappa_l})^2(d+\log(TK/\alpha)) $ and $h=h_{\text{sub}},\lambda_0 = (2\gamma L\CB)^{-1}(\dfrac{p'}{2})^{1/\beta}$, Algorithm~\ref{algo:LDP-Multi} with OLS mechanism $\psi^{OLS}_{t}$ and estimator $\varphi_{t}^{OLS}$ achieve the following regret with probability at least $1-\alpha$ for some constant $C$,
\begin{align*}
	\text{Reg}(T)\leq \gamma C\CB  
	\Big[
		\big(\dfrac{K\CB(\CB \sigma_\epsilon+\pfactorc )\sqrt{d+\log((TK)/\alpha)} }{\kappa_l p'}\big)^{1+\beta}
		+o_{h_{\text{sub}},\beta,\gamma}(1)
	\Big]\cdot
	\left\{
		\begin{array}{ll}
			\log T, &\beta = 1,\\
			\dfrac{T^{\frac{1-\beta}{2}}}{1-\beta}, &0<\beta <1.
		\end{array}
	\right.  
\end{align*}
Under Assumptions~\ref{assump:link}, \ref{diversecondition} and \ref{betamargin}, with the choice of step-size $$\eta_t: = (\bm 1\{t\leq Ks_0 \}((t\text{ mod } K)+1) +\bm 1\{t>Ks_0 \} (t-(K-1)s_0))^{-1}K_{opt}^{-1}\zeta\kappa_lp'c'$$ for any $c'\geq 1$ and $h= h_{\text{sub}}$, Algorithm~\ref{algo:LDP-Multi} with SGD mechanism $\psi^{SGD}_{t}$ and estimator $\varphi_{t}^{SGD}$ achieve the following regret with probability at least $1-\alpha$ for some constant $C$,
\begin{align*}
	\text{Reg}(T)\leq C\cdot \gamma L \CB
	\Big[ 
		\big(
			\dfrac{Kr_{\varepsilon,d} L\CB\sqrt{\log((TK\log T)/\alpha)} }{\zeta\kappa_l p'}
		\big)^{1+\beta} 
		+o_{h_\text{sub},\beta,\gamma}(1)
	\Big]
	\cdot 
	\left\{
		\begin{array}{ll}
			\log T , &\beta = 1,\\
			\dfrac{T^{\frac{1-\beta}{2}}	}{1-\beta}, &0<\beta <1.
		\end{array}
	\right. 
\end{align*}
\end{theorem}

Theorem \ref{mainthmsm} recovers the non-privacy bound in \cite{Bastani2017} under similar condition up to a logarithmic factor. Notice that unlike Theorem \ref{thm:opt-single-margin} in the single-parameter case, we cannot establish the regret when $\beta = 0$. The reason is that in our analysis, we need the probability of $\triangle_t>h$ vanish as $h\to 0$ to guarantee the estimation error for $\theta_{i},i\in K_{opt}$ converges. The corresponding theoretical result in this setting when $\beta=0$ is left as an open question.

\section{Experiment}
\label{sec:numerical}

To the best of our knowledge, the contextual bandit algorithms with LDP guarantee has only been studied by \cite{Zheng2020}, who propose a variant of LinUCB algorithm for linear bandits and a variant of Generalized Linear Online-to-confidence-set Conversion (GLOC) framework \cite{Jun2017} for generalized linear bandits.
We refer their methods as LDP-UCB and LDP-GLOC.
We call our method LDP-OLS if we plug in the OLS mechanism and estimator into Algorithms~\ref{LDP-Single} and \ref{algo:LDP-Multi}, and LDP-SGD if we plug in the SGD ones.
We evaluate all the four methods on two different privacy levels $\varepsilon=1, 5$ in synthetic datasets, which are industry standards. For example, Apple uses $\varepsilon = 4$ in their projects on Emojis and Safari usage \cite{apple17}. Similar choices of the privacy parameter $\varepsilon$ can be found in \cite{bassily2017practical, erlingsson2014rappor}. We also demonstrate the efficacy of our algorithms with real data on auto lending in Appendix~\ref{Real-Experiment}.

For the sake of comparison, the learning step parameter for LDP-GLOC and LDP-SGD are tuned in the same way.\footnote{The source code to reproduce all the results is available at the GitHub repo \href{https://github.com/liangzp/LDP-Bandit}{liangzp/LDP-Bandit}.}. The first and second columns in Figure~\ref{fig:numerical} are for single-param and multi-param settings, respectively, which are simulation studies on linear bandits.
The context is generated from Unif$(S^{d-1}_1)$ at each round.

In conclusion, our methods significantly outperform existing ones in all settings consistently. In particular, LDP-SGD achieves better performance under more strigent privacy requirements.

\begin{figure}[H]
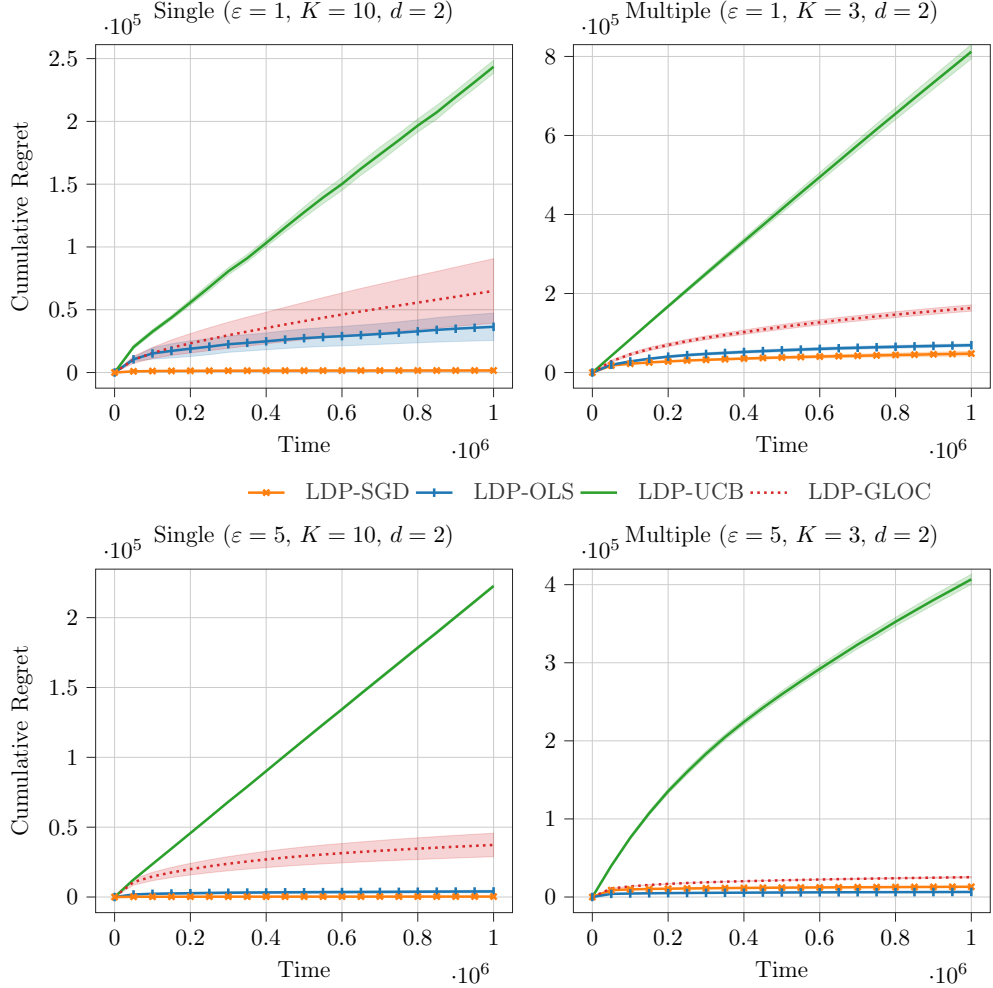

\centering
\begin{tikzpicture}[scale = 0.8]
  \definecolor{LDP-OLS}{rgb}{0.12156862745098,0.466666666666667,0.705882352941177}
  \definecolor{LDP-SGD}{rgb}{1,0.498039215686275,0.0549019607843137}
  \definecolor{LDP-UCB}{rgb}{0.172549019607843,0.627450980392157,0.172549019607843}
  \definecolor{LDP-GLOC}{rgb}{0.83921568627451,0.152941176470588,0.156862745098039}
  
\begin{groupplot}[group style={group size=2 by 2}]

\input{simulation_single_epsilon_1.tex}

\input{simulation_multi_epsilon_1.tex}

\input{simulation_single_epsilon_5.tex}

\input{simulation_multi_epsilon_5.tex}


\end{groupplot}
\end{tikzpicture}
\caption{
  We perform 10 replications for each case and plot the mean and 0.5 standard deviation of their regrets.
  }
\label{fig:numerical}
\end{figure}

\section{Conclusion}
\label{Conclusion}
In this paper, we propose LDP contextual bandit frameworks in both single-parameter and multi-parameter settings with flexibility to deal generalized linear reward structure, and establish theorectical guarrentee of our algorithms based on the frameworks. 
Our algorithms are highly efficient and have superior empirical performance. 
There are still some open questions to be explored.
Whether our regret bounds are optimal in terms of $\varepsilon$ in the multi-parameter setting is still unknown. 
It will be interesting to explore estimators and mechanisms beyond the private OLS and SGD ones to study the optimality in terms of $\varepsilon$.
Moreover, whether there is a fundamental limit in adversarial contextual bandit under LDP constraints is still an open question.
It also remains an open question to analyze the regret bound in the multi-parameter setting when $\beta=0$ in the margin condition.

\bibliography{ref}

\appendix

\section{Randomness Condition}
\label{Randomness}
In this section, we show that a sub-gaussian random vector with bounded density satisfies Assumption~\ref{as:pstar}:\\ We say a random vector $x$ is $\sigma^2$-sub-gaussian vector with bounded density, if for every $v\in S^{d-1}_1$, $v^Tx$ is $\sigma^2$-sub-gaussian and its density function exists and is bounded by $\gamma$ for some $\gamma>0$. For such kind of random vector, \cite{Ren2020a} shows that it satisfies Assumption~\ref{as:pstar} with $\kappa_l =  \dfrac{2d}{3\gamma K}$ and $p_* =  \dfrac{1}{3}$. In particular, \cite{Han2020} shows that when $x$ follows $\mathcal N(0,\Sigma)$, with $\lambda_{\min}(\Sigma) \geq \dfrac{\kappa}{d}$, we can have $\kappa_l = \dfrac{c_1\kappa }{d}$ and $p_* = c_2 $ for constants $c_1$ and $c_2$.

\section{Proof of Privacy Guarantee}
\label{Proof-Privacy}
\subsection{Proof of Results in Section \ref{sec:single-algo}}
\begin{proof}[Proof of Proposition~\ref{privacy:single-OLS}]
	Since we assume that the features and rewards are bounded, $\lVert x_{t,a}\rVert\leq \CB,\lVert r_{t}\rVert\leq \RB$ for all $t\in [T]$ and $a\in [K]$, by Lemma~\ref{lem:gaussian}, $M_t$ is $(\varepsilon/2,\delta/2)$-LDP and $u_t$ is $(\varepsilon/2.\delta/2)$-LDP. Thus Lemma~\ref{lem:composition} implies that $\psi_t^{OLS}$ is $(\varepsilon,\delta)$-LDP.
\end{proof}

\begin{proof}[Proof of Proposition~\ref{privacy:single-SGD}]
	Since we assume that the features and rewards are bounded, $\lVert x_{t,a}\rVert\leq \CB,\lVert r_{t}\rVert\leq \RB$ for all $t\in [T]$ and $a\in [K]$, we have $(\mu(x_{t,a_t}^T\hat{\theta}_{t-1})-r_t)x_{t,a_t}$ bounded by $2c_rC_B$. Lemma~\ref{lem:l2} implies that $\psi_t^{SGD}$ is $\varepsilon$-LDP. 
\end{proof}

\subsection{Proof of Results in Section~\ref{sec:multi-privacy}}

\begin{proof}[Proof of Proposition~\ref{privacy:multi-ols}]  We simply denote $\psi_{t}^{OLS}$ by $\psi_t$ in this proof.
	At time $t$, for any two $x\neq x'$, without loss of generality\ assuming the action corresponding  $x$ and $x'$ are $a_t = 1$ and $a_t = 2$, then the output corresponding $x,x'$ is given by $(\psi_t(x,x^T\theta_1+\epsilon_t),\psi_t(0,0),\dots,\psi_t(0,0))$ and $(\psi_t(0,0),\psi_t(x',x'^T\theta_2+\epsilon_t),\dots,\psi_t(0,0)).$ Since $\psi_t(0,0)$ has the same distribution, we have for any subset $A_1\times A_2\times\dots\times A_K\subset \R^{Kd}$ with $A_{i}$  a Borel set in $ \R^{d}$, 
	\begin{align}
		&\quad\dfrac{\mathbb P(\psi_t(x,x^T\theta_1+\epsilon_t)\in A_{1},\psi_t(0,0)\in A_{2},\dots,\psi_t(0,0)\in A_{K})}{\mathbb P(\psi_t(0,0)\in A_{1},\psi_t(x',x'^T\theta_2+\epsilon_t)\in A_{2},\dots,\psi_t(0,0)\in A_K)}\nonumber\\
		&= \dfrac{ \mathbb P(\psi_t(x,x^T\theta_1+\epsilon_t)\in A_{1},\psi_t(0,0)\in A_{2})}{\mathbb P(\psi_t(x',x'^T\theta_2+\epsilon_t)\in A_{2},\psi_t(0,0)\in A_{1})}.
		\label{eq: lemma E1}
	\end{align}
	Set $\tilde{\psi}(v_1,v_2): = (\psi_t(v_1),\psi_t(v_2))$, and $(v_1,v_2) \coloneqq (x,0), (v'_1,v'_2)\coloneqq (0,x')$, then we have \eqref{eq: lemma E1} equals to $\tilde{\psi}(v_1,v_2)/\tilde{\psi}(v_1',v_2')$, thus applying Lemma~\ref{lem:composition} to it implies that \eqref{eq: lemma E1} is upper bounded  by $e^\varepsilon+\delta \mathbb P(\psi_t(x',x'^T\theta_2+\epsilon_t)\in A_{2},\psi_t(0,0)\in A_{1})^{-1}$, leading to the desired result.
\end{proof}

\begin{proof}[Proof of Proposition~\ref{privacy:multi-sgd}]
That is nearly the same as the proof of Proposition~\ref{privacy:multi-ols}, but replacing $e^\varepsilon+\delta \mathbb P(\psi_t(x',x'^T\theta_2+\epsilon_t)\in A_{2},\psi_t(0,0)\in A_{1})^{-1}$	 by  $e^\varepsilon$ in the last step.
\end{proof}

\section{Proof of Results in Section~\ref{sec:single-regret} }

In the following analysis, without special explaination, all the $c$ and $C$ denote absolute constants. Sometimes we state the inequality of type $A_1\leq C\log(A_2/\alpha) A_3$ holds with probability at least $1-\alpha$ while in proof we derive the results hold with $1-c\alpha$ for some constant c. In fact, they are equivalent by re-scaling $\alpha$ and changing $C$ to some larger constant.

\subsection{Proof of Worst-Case Bounds}

\begin{proof}[Proof of Theorem~\ref{ub-single-worst}]
	Since $x_{t,a_t}$ is the greedy selection, we have $x_{t,a_t}^T\hat{\theta}_{t-1}\ge x_{t,a}^T\hat{\theta}_{t-1}$ for any time $t\in [T]$ and $a\in [K]$. Consequently we have the following upper bound for the instantaneous regret at time $t$,
	\begin{align*}
		\max _{a \in[K]}\left(x_{t, a}-x_{t, a_{t}}\right)^{T} \theta^{\star} & \leq \max _{a \in[K]}\left(x_{t, a}-x_{t, a_{t}}\right)^{T}\left(\theta^{\star}-\hat{\theta}_{t-1}\right)\\
		& \leq \max _{a, a^{\prime} \in[K]}\left(x_{t, a}-x_{t, a^{\prime}}\right)^{T}\left(\theta^{\star}-\hat{\theta}_{t-1}\right)\\
		& \leq 2 \max _{a \in[K]}\left|x_{t, a}^{T}\left(\theta^{\star}-\hat{\theta}_{t-1}\right)\right|.
	\end{align*}

	 For any fixed $a\in [K]$, $x_{t,a}$ is independent of $\hat{\theta}_{t-1}$. By Assumption~\ref{as:pstar}, conditioning on the historical information up to time t,  $ x_{t,a}^T (\theta^\star-\hat{\theta}_{t-1})$ is a $\frac{\kappa_u}{d} \lVert \theta^\star-\hat{\theta}_{t-1}\rVert^2$-sub-gaussian random variable. Now by the maximal concentration inequality for a sub-gaussian sequence, we have with probability at least $1-\frac{\alpha}{T}$, 
	 \begin{align*}
		\max_{a\in [K]}\lvert  x_{t,a}^T (\theta^\star-\hat{\theta}_{t-1})\rvert = O \left ( \sqrt{\dfrac{\kappa_u \log (KT/\alpha)}{d}} \lVert\theta^\star -\hat\theta_{t-1}\rVert \right ).
		\end{align*}
	To control the regret bound, we  bound the estimation error $\lVert \theta^{\star}-\hat{\theta}_{t-1}\rVert$ in each time in the following lemma. 

	\begin{lemma}[Estimation Error for OLS]
		\label{est-ols}
		Using the private OLS update mechanism $\psi^{OLS}_{t}$ and estimator $\varphi_{t}^{OLS}$, for any $8 \frac{d \log 9 +\log (T/\alpha)}{p_{*}^2}< t \le T$, we have with probability at least $1-\dfrac{\alpha}{T}$, 
			 \begin{equation}
				\lVert \hat{\theta}_t-\theta^{\star}\rVert^2 \leq C (\CB \sigma_{\epsilon} \pfactorc d)^2  \dfrac{d+\log(T/\alpha) }{\kappa_l^2 p_{*}^2 t},
			\label{ols-est-error}
			\end{equation}
	for some C independent of d, K and T.
	\end{lemma}

	\begin{lemma}[Estimation Error for SGD]\label{est-sgd}
		Using the private OLS update mechanism $\psi^{SGD}_{t}$ and estimator $\varphi_{t}^{SGD}$, for any $3\le t\le T$, we have with probability at least $1-\dfrac{\alpha}{T}$, 
		\begin{equation}
			\lVert \hat{\theta}_t - \theta^{\star} \rVert^2 \le \frac{(624\log (\log T/\alpha)+1)r_{\varepsilon, d}^2d^2}{4\kappa_l^2 \zeta^2 p_{*}^2 t}.
		\label{sgd-est-error}
		\end{equation}
	\end{lemma}

Plugging OLS estimation error \eqref{ols-est-error} into the regret bound, denote $t_1 \coloneqq 8 \frac{d \log 9 +\log (T/\alpha)}{p_{*}^2}$, the following holds with probability at least $1-\alpha$,
\begin{align}
	&\sum_{t=1}^T \max _{a \in[K]}\left(x_{t, a}-x_{t, a_{t}}\right)^{T} \theta^{\star}\nonumber\\
	\le &t_1 \RB + \sum_{t=t_1+1}^T C \CB \sigma_{\epsilon}\pfactorc d \sqrt{\frac{\kappa_u \log (KT/\alpha)}{d}}\dfrac{\sqrt{d+\log(T/\alpha)}}{\kappa_l p_{*} \sqrt{t}}\label{REGBOUND1}\\
	\le & 8 \frac{d \log 9 +\log (T/\alpha)}{p_{*}^2} + C \CB \pfactorc \sigma_{\epsilon} \sqrt{d} \dfrac{\sqrt{d+\log(T/\alpha)}}{\kappa_l p_{*}}\sqrt{\kappa_u \log (KT/\alpha)}\sqrt{T}.\nonumber
\end{align} 

Plugging the SGD estimation error \eqref{sgd-est-error} into the regret bound, we have
\begin{align}
	&\sum_{t=1}^T \max _{a \in[K]}\left(x_{t, a}-x_{t, a_{t}}\right)^{T} \theta^{\star}\nonumber\\ 
	\le & 2\RB + \sum_{t=3}^T \sqrt{\kappa_u \log (KT/\alpha)} \frac{\sqrt{(624\log (\log T/\alpha)+1)}r_{\varepsilon, d}\sqrt{d}}{2\kappa_l \zeta p_{*} \sqrt{t}}\nonumber\\
	\le & 2\RB + \frac{\sqrt{(624\log (\log T/\alpha)+1)}r_{\varepsilon, d}\sqrt{d}}{2\zeta \kappa_l p_{*}} \sqrt{\kappa_u \log (KT/\alpha)} \sqrt{T}.\label{REGBOUND2}
\end{align} 
\end{proof}

So now it suffices to prove the Lemmas~\ref{est-ols} and \ref{est-sgd}.

\subsection{Proof of lemma \ref{est-ols}}
\begin{lemma}
    As long as $t>8 \frac{d \log 9 +\log (T/\alpha)}{p_{*}^2}$, the following lower bound
    \begin{align*}
    	 \lambda_{\min} (\sum_{i=1}^t x_{i, a_i}x_{i, a_i}^T) \geq C\cdot \dfrac{t\kappa_lp_{*}}{d},
    \end{align*}
   holds  with probability at least $1-\dfrac{\alpha}{T}$, for some $C$ independent of $d$ and $T$.
    \label{smallest}
\end{lemma}
\begin{proof}
	Define  $\mathcal F^{-}_t$ as the filtration generated by $\{x_{i,a_i}\}_{i\in [t-1]}$, $\{\epsilon_{i}\}_{i\in[t-1]}$ and the randomness from $\{\psi^{OLS}_i\}_{i\in [t-1]}$. By greedy algorithm, in each time $i$, $x_{i, a_i}$ is selected as $a_i= \text{argmax}_{a\in [K]} x_{i,a}^T\hat{\theta}_{i-1}$. Thus by the Assumption~\ref{as:pstar},  we have for any $0<s<p_{*}$, \begin{align*}
    &\mathbb P(\sum_{i=1}^t (x_{i, a_i}^T v)^2< {t\kappa_l (p_*-s)/d})  \\
	&\leq   \mathbb P(\sum_{i=1}^t\bm 1\{(x_{i, a_i}^Tv)^2 >\kappa_l/d\}< {t (p_*-s)})\\
    &\leq \mathbb P(\dfrac{1}{t}\sum_{i=1}^t(\bm 1\{(x_{i, a_i}^Tv)^2 >\kappa_l/d\}-\E[ \bm 1\{(x_{i, a_i}^Tv)^2 >\kappa_l/d\} \lvert \mathcal F^{-}_{i}]))< -s)\\
    &\leq \exp(-\dfrac{s^2t}{2}),
\end{align*}
where in the last inequality we use the Azuma–Hoeffding's inequality for bounded martingale-difference sequence (see Corollary~2.20 in \cite{wainwright2019high}). 

   For every $d\times d$ positive-definite matrix $A$, with an abuse of notation, we denote $\mathcal N_\varepsilon$ as the $\varepsilon$-net of $S^{d-1}_1$ for some $\varepsilon>0$ to be determined,
    \begin{align*}
    	\lambda_{\max}(A)\leq \dfrac{1}{1-2\varepsilon} \sup_{x\in \mathcal N_\varepsilon} x^TAx,
    \end{align*}
    which then implies    
    \begin{align*}
    	\lambda_{\min}(A) = -\lambda_{\max}(-A)\geq \dfrac{-1}{1-2\varepsilon}\sup_{x\in \mathcal N_\varepsilon} x^T(-A)x = \dfrac{1}{1-2\varepsilon}\inf_{x\in \mathcal N_\varepsilon} x^TAx. 
    \end{align*} By choosing $\varepsilon = 1/4$, we can find an $\varepsilon$-net $\mathcal{N}_{\varepsilon}$ with cardinality $\lvert \mathcal{N}_{\varepsilon} \rvert\le 9^d$. Therefore
    \begin{align*}
    	\lambda_{\min} (A) \ge 2 \inf_{x\in \mathcal N_\varepsilon} x^TAx.
    \end{align*}
	Note that
    \begin{align*}
    	\mathbb P(\min_{\lVert v\rVert = 1}\sum_{i=1}^t (x_{i,a_i}^Tv)^2<2 t\kappa_l (p_{*}-s)/d) 
    	&\leq \mathbb P(\sum_{i=1}^t (x_{i,a_i}^Tv)^2< t\kappa_l (p_{*}-s)/d,\exists v\in \mathcal{N}_\varepsilon)\\
    	&\leq  9^d \exp(-\frac{s^2t}{2}).
    \end{align*}
    
    By setting $s= \sqrt{\frac{2 d \log 9 +2\log ({T}/{\alpha})}{t}}$, we have when $t>8 \frac{d \log 9 +\log (T/\alpha)}{p_{*}^2}$ with probability at least $1-\dfrac{\alpha}{T}$, 
		\begin{align*}
    	\lambda_{\min} (\sum_{i=1}^t x_{i, a_i}x_{i, a_i}^T) = \min_{\lVert v\rVert =1 }\sum_{i=1}^t\langle x_{i, a_i},v\rangle^2 \geq \dfrac{\kappa_l p_* t }{d}.
		\end{align*}
\end{proof}

\begin{proof}[Proof of Lemma~\ref{est-ols}]\quad\\
	By lemma \ref{smallest} we know that with probability at least $1-\dfrac{\alpha}{T}$,  
    \begin{align*}
    	 \lambda_{\min} (\sum_{i=1}^t x_{i, a_i}x_{i, a_i}^T) \geq C_1 \kappa_l p_{*} t/d,
    \end{align*}
    for some $C_1$ independent of $d, K$ and $T$.	
	
	Since $\{W_i\}_{i\in[t]}$ are independent, therefore by concentration bounds for Wigner matrix we have with probability at least $1-\frac{\alpha}{T}$,
	\begin{align*}
    	\lVert \sum_{i=1}^{t} W_i \rVert^2 \le C_2t\pfactorc^2 (d + \log (T/\alpha)),
    \end{align*}
    for some $C_2$ independent of $d, K$ and $T.$ However, it is important to note that the perturbation of privacy noise matrix $\sum_{i=1}^{t} W_i$ may destroy the positive definite property of the Gram matrix $\sum_{i=1}^t x_{i,a_i}x_{i,a_i}^T$ when t is still small. Therefore, we shift $\sum_{i=1}^{t} W_i$ by adding $\tilde{c}  \sqrt{t} I_d$ where $\tilde{c} \coloneqq C_2\pfactorc (\sqrt{d} + \sqrt{\log (T/\alpha)})$.
    
    We denote $A_t \coloneqq \sum_{i=1}^t (x_{i, a_i} x_{i, a_i}^T + W_i) + \tilde{c} \sqrt{t}I$. Therefore, by Weyl's inequality we have with probability at least $1-\frac{\alpha}{T}$,
    \begin{align*}
    	\lambda_{\min}(A_t) = \lambda_{\min}\left ( \sum_{i=1}^t (x_{i, at_i} x_{i, a_i}^T + W_i) + \tilde{c} \sqrt{t} I_d\right )\ge \lambda_{\min}\left ( \sum_{i=1}^t x_{i, a_i}x_{i, a_i}^T \right ) \ge C_1 \kappa_l p_{*}t/d.
    \end{align*}
	So now we we study the OLS estimator with $x_{i,a_i},\epsilon_i$ given above and $r_i=x_{i,a_i}^T\theta^{\star}+\epsilon_i$. In that case, the estimation error of the OLS estimator under LDP constraints at time $t$ is given by 
	\begin{align*}
    	\hat{\theta}_t - \theta^{\star} &= A_t^{-1} \sum_{i=1}^t(x_{i, a_i} r_i+ \xi_i) - \theta^{\star}\\
    	&= A_t^{-1}\sum_{i=1}^t (x_{i, a_i} x_{i, a_i}^T\theta^{\star} + x_{i, a_i}\epsilon_i + \xi_i) - \theta^{\star}\\
    	& = A_t^{-1}  ( \sum_{i=1}^tx_{i, a_i} \epsilon_i ) - A_t^{-1} \sum_{i=1}^t W_i \theta^{\star}+ A_t^{-1} \sum_{i=1}^t \xi_i -  \tilde{c} \sqrt{t} A_t^{-1}\theta^{\star}.
  \end{align*}

  Define  $\mathcal F_t$ as the filtration generated by $\{x_{i,a_i}\}_{i\in [t]}$, $\{\epsilon_{i}\}_{i\in[t-1]}$ and the randomness from $\{\psi_i\}_{i\in [t-1]}$. Notice that for every unit vector $u$, \begin{align*}
	\E[\exp(\lambda \sum_{i=1}^t u^Tx_{i, a_i}\epsilon_i )] &=
	\E[\E [\exp(\lambda \sum_{i=1}^t u^Tx_{i, a_i}\epsilon_i)\lvert \mathcal F_{t}]]\\
	&=\E[\prod_{i=1}^{t-1}\exp(\lambda  u^Tx_{i, a_i}\epsilon_i )\E [\exp(\lambda  u^TX_i\epsilon_i )\lvert \mathcal F_{t}]]\\
	&\overset{(1)}{\leq} \exp(\dfrac{\lambda^2\CB^2\sigma_{\epsilon}^2}{2}) \E[\prod_{i=1}^{t-1}\exp(\lambda  u^Tx_{i, a_i}\epsilon_i )]\\
	&\overset{(2)}{\leq} \exp(\dfrac{\lambda^2\CB^2\sigma_{\epsilon}^2t}{2}).
\end{align*}
Inequality (2) is due to the mathematical induction using the same technique in the equality (1). Thus $\sum_{i=1}^t x_{i, a_i}\epsilon_i$ is $\sigma^2\CB^2t$-sub-gaussian vector, and by the concentration of norm for sub-gaussian vectors, we have then with probability at least $1-\frac{\alpha}{T}$, 
\begin{align*}
	\lVert \sum_{i=1}^tx_{i,a_i}\epsilon_i\rVert^2\leq C_3\sigma_{\epsilon}^2\CB^2t (d+ \log(T/\alpha)),
\end{align*}
    where $C_3$ is a positive constant independent of $d$, $K$ and $T$.
    
    Therefore, 
\begin{align}
	\begin{aligned}
    	\lVert A_t^{-1}  ( \sum_{i=1}^tx_{i, a_i} \epsilon_i )\rVert^2 & \le \lVert A_t^{-1}\rVert^2 \lVert  ( \sum_{i=1}^tx_{i, a_i} \epsilon_i )\rVert^2\\
    	& \le \frac{C_3\sigma^2\CB^2d^2t (d+ \log(T/\alpha))}{(C_1\kappa_l p_{*} t)^2}.
		\label{ols-part1}
	\end{aligned}
\end{align}
Moreover, 
\begin{align}
	\begin{aligned}
	\lVert A_t^{-1} \sum_{i=1}^t W_i \theta^{\star} \rVert^2 & \le \lVert A_t^{-1} \rVert^2 \lVert \sum_{i=1}^t W_i\rVert^2 \lVert \theta^{\star} \rVert^2\\
	& \le \lVert A_t^{-1} \rVert^2 \lVert \sum_{i=1}^t W_i\rVert^2\\
	& \le \frac{C_2t\pfactorc^2 (d + \log (T/\alpha))}{(C_1\kappa_l p_{*} t)^2},
	\label{ols-part2}
	\end{aligned}
\end{align}
where the second inequality is from the assumption that $\lVert \theta^{\star} \rVert \le 1$.
	
	Third, Since $\xi_i$ are random vector with independent, sub-gaussian coordinates that satisfy $\mathbb{E}\xi_{i,j}^2 = \pfactorc^2$, $\sum_{i=1}^t \xi_i$ is a random vactor with independent sub-gaussian coordinates that satisfy $\mathbb{E}\sum_{i=1}^t \xi_{i,j}^2 = t\pfactorc^2$. Therefore for all $t \in [T]$, with probability at least $1-\frac{\alpha}{T}$,
    \begin{align*}
    	\lVert \sum_{i=1}^t \xi_i  \rVert^2\le C_4 t\pfactorc^2 (d + \log (T/\alpha)),
	\end{align*}
	for some positive constant $C_4$ independent of $d$, $K$ and $T$.
    Therefore, 
	\begin{align}
		\begin{aligned}
    	\lVert A_t^{-1} \sum_{i=1}^t \xi_i \rVert^2 \le \frac{C_4 t\pfactorc^2d^2 (d + \log (T/\alpha))}{(C_2\kappa_l p_{*} t)^2}.
		\label{ols-part3}
		\end{aligned}
    \end{align}

	Lastly, 
	\begin{align}
		\begin{aligned}
			\lVert \tilde{c} \sqrt{t} A_t^{-1}\theta^{\star}\rVert^2\le \dfrac{\tilde{c}^2t}{(C_2\kappa_l p_{*} t)^2},
			\label{ols-part4}
		\end{aligned}
	\end{align}
	 holds with probability at least $1-\frac{\alpha}{T}$. Plugging all bounds \eqref{ols-part1} \eqref{ols-part2} \eqref{ols-part3} and \eqref{ols-part4} together we get then with probability at least $1- \frac{\alpha}{T}$,
    \begin{align*}
    	\lVert \hat{\theta}_t-\theta^{\star}\rVert^2\leq C_5 \sigma_{\epsilon}^2 \CB^2 \pfactorc^2d^2 \dfrac{d+\log(T/\alpha) }{\kappa_l^2 p_{*}^2 t},
    \end{align*}
    for some positive constant $C_5$ independent of $d, K$ and $T.$
\end{proof}

\begin{proof}[Proof of Lemma~\ref{est-sgd}]
	Denote $g_t$ as the gradient at time t, $\hat{g}_{t}\coloneqq \Psi_{\varepsilon}[(\mu(x_{t,a_t}^T\hat\theta_{t})-r_t)x_{t,a_t}]$ is the LDP private estimator of $g_t$ and $\hat{z}_t = g_t - \hat{g}_t$. By the unbiasedness of $\Psi_{\varepsilon}$ in Lemma~\ref{privacy-bounded-vector} we have 
	\begin{align*}
		& \E[ \Psi_\varepsilon((\mu(x_{t,a_t}^T\hat\theta_{t-1})-r_t) x_{t,a_t})^T(\hat\theta_{t-1}-\theta^{\star})  \lvert \mathcal F_{t-1}] \\
		 = & \E[(\mu(x_{t,a_t}^T\hat\theta_{t-1})-\mu(x_{t,a_t}^T\theta^{\star}))x_{t, a_t}^T(\hat\theta_{t-1}-\theta^{\star} )\lvert \mathcal F_{t-1} ]\\
		\geq & \zeta \E[[x_{t,a_t}^T(\hat\theta_{t-1}-\theta^{\star})]^2 \lvert \mathcal F_{t-1} ]\geq \zeta \kappa_l p_*/d \lVert \hat\theta_{t-1}-\theta^{\star}\rVert^2, 
	\end{align*}
	where the last inequality is from Lemma \ref{smallest} and Markov's inequality $\lambda_{\min}(\E_{x_{a_t}}[x_{a_t}x_{a_t}^T \lvert \mathcal F_{t-1}])\ge \kappa_l p_{*}/d$.
	Moreover, notice that $\lVert \hat{g}_t \rVert= r_{\varepsilon,\delta}$. 
Let $\lambda \coloneqq 2 \kappa_l \GB p_{*}/d$ and $\eta_t = \frac{1}{\lambda t}$, 
\begin{align*}
	\lVert \hat{\theta}_{t}-\theta^{\star}\rVert^2 &= \lVert \hat{\theta}_{t-1}-\eta_t\hat g_t-\theta^{\star}\rVert^2\\
	& =  \lVert  \hat{\theta}_{t-1}-\theta^{\star}\rVert^2-2\eta_t \hat{g}_t^T (\hat{\theta}_{t-1}-\theta^{\star})+\eta_t^2 \lVert \hat g_t\rVert^2\\
	& = \lVert  \hat{\theta}_{t-1}-\theta^{\star}\rVert^2-2\eta_t {g}_t^T (\hat{\theta}_{t-1}-\theta^{\star})+ 2\eta_t \hat{z}_t^T (\hat{\theta}_{t-1}-\theta^{\star}) + \eta_t^2 \lVert \hat g_t\rVert^2\\
	& \le (1-2\lambda \eta_t )\lVert  \hat{\theta}_{t-1}-\theta^{\star}\rVert^2 + 2\eta_t \hat{z}_t^T (\hat{\theta}_{t-1}-\theta^{\star}) + \eta_t^2 \lVert \hat g_t\rVert^2\\
	& \le \left(1-\frac{2}{t}\right)\lVert  \hat{\theta}_{t-1}-\theta^{\star}\rVert^2+\frac{2}{\lambda t} \hat{z}_t^(\hat{\theta}_{t-1}-\theta^{\star})+\left(\frac{r_{\varepsilon,d}}{\lambda t}\right)^{2} .
\end{align*}
It follows from the same proof as in Proposition~1 in \cite{rakhlin2011making}, we can obtain for any $0<\alpha\le \frac{1}{eT}$, $T\ge 4$ and for all $3 \le t\le T$, with probability at least $1-\alpha$, 
\begin{align*}
	\lVert \hat{\theta}_t - \theta^{\star} \rVert^2 \le \frac{(624\log (\log (T)/\alpha)+1)r_{\varepsilon, d}^2 d^2}{4\kappa_l^2 \zeta^2 p_{*}^2 t}.
\end{align*}
\end{proof}

\subsection{Proof of Problem-dependent Bound}
To prove the problem-dependent bound, we need only combine Lemma~\ref{est-ols} and Lemma~\ref{est-sgd} together with the following lemma.
\begin{lemma} Under the $(\beta,\gamma)$-margin condition, if we have $\lVert \hat{\theta}_t-\theta^{\star}\rVert\leq \dfrac{U_0}{\sqrt{t}}$ holds uniformly for all $t_0\leq t\leq T_0$ for some $t_0$ and $U_0$ with probability at least $1-\alpha$, we have then with probability at least $1-2\alpha$,
	\begin{align*}
	\text{Reg}(T)\leq C\cdot  \left\{\begin{array}{ll}
		c_rt_0+\gamma (LC_BU_0)^2(\log T + o(1)), & \text{ }\beta = 1\\
		c_rt_0+\frac{2\gamma}{1-\beta}(L\CB U_0)^{1+\beta}(T^{\frac{1-\beta}{2}}+o(1)),	&\text{ }0\leq \beta<1.
	\end{array}\right.
\end{align*}
\end{lemma}
\begin{proof}
	We have, with probability at least $1-\alpha$, 
	\begin{align*}
		\text{Reg}(T)&\leq  2\RB t_0 + (\mu(x_{t,a_t^*}^T\theta^{\star})-\mu(x_{t,a_t}^T\theta^{\star})) \bm 1\{\lVert \hat{\theta}_t-\theta^{\star}\rVert\leq \dfrac{U_0}{\sqrt{t}}, \triangle_t\leq \dfrac{2L\CB U_0}{\sqrt{t}}  \}\\
		&\leq 2\RB t_0+2L\CB\dfrac{U_0}{\sqrt{t}}\bm 1\{\triangle_t\leq \dfrac{2L\CB U_0}{\sqrt{t}} \}.
	 \end{align*}
	 Denote $A_t: = \dfrac{1}{\sqrt{t}} \bm 1\{\triangle_t\leq \dfrac{2L\CB U_0}{\sqrt{t}} \}$, by Hoeffding's inequality we have with probability at least $1-\alpha$, 
	 \begin{align*}
	 	\sum_t A_t<\sum_t\E[A_t]+\sqrt{\log T\log \dfrac{1}{\alpha}}.
	 \end{align*}
	 Noting that $\E[\sum_tA_t] \leq  2\gamma L\CB U_0\log T $ for $\beta = 1$ and $\E[\sum_tA_t]\leq  \dfrac{2\gamma}{1-\beta} (L\CB U_0 )^{\beta}T^{\frac{1-\beta}{2}} $ for $0\le \beta<1$. Then the claim holds.
\end{proof}

\section{Proof of Results in Section \ref{sec:multi-regret}}
To lighten the notation, in this section we denote  $\theta_i$ the underlying parameter of arm i. In the following analysis, without special explaination, all the $c$ and $C$ denote absolute constants. Sometimes we state the inequality of type $A_1\leq C\log(A_2/\alpha) A_3$ holds with probability at least $1-\alpha$ while in proof we derive the results hold with $1-c\alpha$ for some constant c. In fact, they are equivalent by re-scaling $\alpha$ and changing $C$ to some larger constant.

\subsection{Proof of Theorem~\ref{mainthmsm} }
\begin{lemma}
	\label{warmupstageh} 
	If after the warm up stage of length $Ks_0$, the estimator $\hat{\theta}_{Ks_0,i}$ achieves the following error bound with probability at least $1-\alpha$,
	\begin{align*}
		\sup_{i\in [K]}\lVert \hat{\theta}_{Ks_0,i}-\theta_{i}\rVert\leq h_0\coloneqq  \dfrac{h_{sub}}{8L\CB},
	\end{align*}
	With $h=h_{sub}$ in Algorithm~\ref{algo:LDP-Multi}, we have $\mathbb P\{a_t^*\in\hat{K}_t,\hat{K}_t\cap K_{sub}=\emptyset\}\geq 1-\alpha$ holds uniformly for all $Ks_0<t\le T.$	
\end{lemma}
\begin{proof}
	Firstly, to show  $a_t^*\in \hat{K}_t$,  without loss of generality we assume that $a_t^*\neq 1$,  and $\text{argmax}_{i\in [K]}\mu(X_t^T\hat{\theta}_{Ks_0,i}) = 1$. Then by the optimality of $\theta_{a_t^*}$, condition on $\sup_{i\in [K]}\lVert \hat{\theta}_{Ks_0,i}-\theta_{i}\rVert\leq h_0$,
	\begin{align*}
		\mathbb{P}(a_t^*\notin \hat{K}_t) &= \mathbb{P}(\mu(X_t^T\hat\theta_{Ks_0,a_t^*} )<\mu(X_t^T\hat\theta_{Ks_0,1})-h/2 )\\
		&\leq \mathbb{P}(\mu(X_t^T\theta_{a_t^*} )-h/8<\mu(X_t^T\theta_{1})+h/8-h/2 ) = 0.
	\end{align*} 
	Now for any $j\in K_{sub}$, we have condition on $\sup_{i\in [K]}\lVert \hat{\theta}_{Ks_0,i}-\theta_{i}\rVert\leq h_0$,
	\begin{align*}
		\mathbb{P}(j\in \hat{K}_t) &\leq  \mathbb{P}(\mu(X_t^T\hat\theta_{Ks_0,a_t^*} )-h/2<\mu(X_t^T\hat\theta_{Ks_0,j}) )\\
		&\leq \mathbb{P}(\mu(X_t^T\theta_{a_t^*} )-3h/4<\mu(X_t^T\theta_{j})+h/4 ) = 0,
	\end{align*}
	where the final equation is due to the sub-optimality gap assumed in Assumption~\ref{diversecondition}.
\end{proof}

\begin{proof}[Proof of Theorem~\ref{mainthmsm}]
We first show the following lemma, which converts the regret bound under margin condition to the estimation error bound:
\begin{lemma}
	\label{marginlemma}  
	Under the $(\beta,\gamma)$-margin condition, given $h_0$ defined in Lemma~\ref{warmupstageh}, suppose there exists some $s_0$ such that with a warm up stage of length $Ks_0$, $\sup_{i\in [K]}\lVert \hat{\theta}_{t,i}-\theta_i\rVert\leq h_0$,  and there exists some $t_0,U_0(\alpha)$ such that with probability at least $1-\alpha$, 
	\begin{align*}
	\sup_{i\in K_{opt}} \lVert \hat{\theta}_{t,i}-\theta_i\rVert \leq \dfrac{U_0(\alpha) }{\sqrt{t}}, \quad \forall t_0\leq t\leq T.
\end{align*}
Then, we have with probability at least $1-2\alpha$, for some constant $C$, 
\begin{align*}
	\text{Reg}(T)\leq C \cdot \left\{
		\begin{array}{ll}
			\RB t_0+ \gamma (L\CB  U_0(\alpha))^2(\log T +o(1)), & \beta = 1\\
			\RB  t_0+ \dfrac{\gamma}{1-\beta} (L\CB {U_0(\alpha)})^{1+\beta} (T^{\frac{1-\beta}{2}} +o(1)), & 0< \beta <  1.
		\end{array}
\right.
\end{align*}
\end{lemma}
\begin{proof}[Proof of Lemma~\ref{marginlemma}]
Denoting $E_t: = \{\hat{K}_t\cap K_{sub} = \emptyset , a_t^*\in \hat{K}_t\}$, we have with probability at least $1-\alpha$, \begin{align*}
	\text{Reg}(T) &\leq 2\RB t_0+ L\sum_{t_0<t\leq T} X_t^T(\theta_{a_t^*}-\theta_{a_t})\\
	&\leq 2\RB t_0+ L\sum_{t_0<t\leq T} X_t^T(\theta_{a_t^*}-\theta_{a_t})\bm 1\{\sup_{i\in K_{opt}}\lVert \hat{\theta}_{t,i}-\theta_{i}\rVert\leq \dfrac{U_0(\alpha)}{\sqrt{t}}, E_t\}\\
	&\leq 2\RB t_0+ L\sum_{t_0<t\leq T} X_t^T(\theta_{a_t^*}-\theta_{a_t})\bm 1\{\sup_{i\in K_{opt}}\lVert \hat{\theta}_{t,i}-\theta_{i}\rVert\leq \dfrac{U_0(\alpha)}{\sqrt{t}},\triangle_t\leq\dfrac{2L\CB U_0(\alpha) }{ \sqrt{t}}, E_t\}\\
	&\leq 2\RB t_0+ L\sum_{t_0<t\leq T} \dfrac{2\CB U_0(\alpha)}{\sqrt{t}}\bm 1\{\triangle_{t}\leq \dfrac{2L\CB U_0(\alpha)}{\sqrt{t}} \}.
\end{align*}
Let $A_t =\bm 1\{\triangle_{t}<\dfrac{2L\CB U_0(\alpha)}{\sqrt{t}} \}$. Then $A_t$ is a sequence of independent 0-1 valued random variable such that $\mathbb{P}(A_t = 1)\leq \gamma (\dfrac{2L\CB U_0(\alpha)}{\sqrt{t}})^\beta$. Then Hoeffding's inequality implies  with probability at least $1-\alpha$, 
\begin{align*}
	\sum_{t_0\leq t\leq T}\dfrac{1}{\sqrt{t}}A_t\leq\E[\sum_{1\leq t\leq T}\dfrac{1}{\sqrt{t}}A_t]+ \sqrt{ \log T\cdot \log(\dfrac{1}{\alpha}}).
\end{align*}
Notice that $\E[\sum_{1\leq t\leq T}\dfrac{1}{\sqrt{t}}A_t]\leq  CL\CB \gamma U_0(\alpha) \log T $ when $\beta = 1$ and   $\E[\sum_{1\leq t\leq T}\dfrac{1}{\sqrt{t}}A_t]\leq   C \dfrac{\gamma}{1-\beta} ({L\CB U_0(\alpha)})^\beta  T^{\frac{1-\beta}{2}}$ when $0<\beta<1$. This completes the proof.
\end{proof}

Given Lemma~\ref{marginlemma}, we need only show that for both the private OLS estimator and the private SGD estimator, we can find the corresponding $s_0,t_0$ and $U_0(\alpha)$.

\begin{lemma}[Result of OLS estimator]\label{OLSTHM}
	Given $h_{0} = \dfrac{h_{sub}}{8L\CB}$ and $\lambda_0 =( 2 L\CB )^{-1}(\dfrac{p' }{2\gamma})^{1/\beta}$, under the $(\beta,\gamma)$-margin condition ,  
	\begin{align*}
		s_0&=CK(\dfrac{\CB \sigma_\epsilon+\pfactorc}{\min\{\lambda_0,h_0\} p'\kappa_l  })^2(d+\log(TK/\alpha)),\\
		t_0&=2Ks_0,\\
		U_0(\alpha)&= \dfrac{K(\CB\sigma_\epsilon+\sigma_{\varepsilon,\delta})\sqrt{d+\log(TK/\alpha)}}{\kappa_l p'}.\\
	\end{align*}
	satisfy the requirements in Lemma~\ref{marginlemma}.
\end{lemma}

\begin{lemma}[Result of SGD estimator]\label{SGDTHM}
	Given $h_{0} = \dfrac{h_{sub}}{8L\CB}$ and $\lambda_0 =( 2 L\CB )^{-1}(\dfrac{p' }{2\gamma})^{1/\beta}$, under the $(\beta,\gamma)$-margin condition, 
	\begin{align*}
		s_0&=  C \left (\dfrac{K r_{\varepsilon,d}}{ \zeta \kappa_l p'  \min\{\lambda_0,  h_0 \}} \right)^2 \log(KT\log(KT)/\alpha),\\
		t_0&= Ks_0+1,\\
		U_0(\alpha)&= C\dfrac{K\sqrt{\log((KT\log KT)/\alpha)}r_{\varepsilon,d} }{\zeta \kappa_l p'},
	\end{align*}
	satisfy the requirements in Lemma~\ref{marginlemma}.
\end{lemma}
Then Theorem~\ref{mainthmsm} follows from combining Lemma~\ref{marginlemma}, \ref{OLSTHM} and \ref{SGDTHM} .
\end{proof}

The proof of Lemma~\ref{OLSTHM} and Lemma~\ref{SGDTHM} needs the following result: For a fixed $\beta\in (0,1]$, we define $h_{0} = \dfrac{h_{sub}}{8L\CB}$,$\lambda_0 = ( 2 L\CB )^{-1}(\dfrac{p' }{2\gamma})^{1/\beta},$  $A_t: =\{\sup_{i\in K_{opt}}\lVert \hat{\theta}_{t,i}-\theta_i\rVert\leq \lambda_0\},H_0: = \{\sup_{i\in [K]} \lVert \hat{\theta}_{Ks_0,i}-\theta_i\rVert\leq h_0 \}$. 
\begin{lemma}\label{Eigenvalue condition} Define $\mathcal F_{t}$ the  filtration generated by $\{X_i\}_{i\in[t]}$,$\{\epsilon_i\}_{i\in [t]}$ together with all randomness from $\{\psi_{i}\}_{i\in[t]}$.  Then we have:
	\begin{align*}
		\lambda_{\min}(\E[X_tX_t\bm 1\{a_t = i\}\lvert \mathcal F_{t-1}])\geq \dfrac{p' \kappa_l}{2K}\bm 1_{A_{t-1}}\bm 1_{H_0} ,\quad \forall i\in K_{opt}.
	\end{align*}
\end{lemma}
\begin{proof}[Proof]
	We have for every unit vector $v$ 
	\begin{align*}
		& \E[v^TX_tX_t^Tv\bm 1\{a_t = i\}\lvert \mathcal F_{t-1}] \\
		\geq  & \bm 1_{H_0} \dfrac{\kappa_l}{K}\E[ \bm 1 \{ \lvert v^TX\bm 1\{X_t\in U_i\}\rvert^2\geq \kappa_l/K,a_t = i, A_{t-1}\} \lvert \mathcal F_{t-1}]\\
	\geq & \bm 1_{H_0} \bm 1_{A_{t-1}} \dfrac{\kappa_l}{K}\E[ \bm 1 \{ \lvert v^TX\bm 1\{X_t\in U_i\} \rvert^2\geq \kappa_l/K\}-\bm{1} \{a_t \neq  i,X_t\in U_i, A_{t-1} \}  \lvert \mathcal F_{t-1}]\\
	\geq &\bm 1_{H_0} \bm 1_{A_{t-1}} \dfrac{\kappa_l}{K}[p'-\mathbb{P}(\{a_t \neq  i,X_t\in U_i \} \cap H_0 \cap A_{t-1} \lvert \mathcal F_{t-1})].
	\end{align*}
	\begin{align*}
		\mathbb{P}(\{a_t \neq  i,X_t\in U_i \}\cap H_0 \cap A_{t-1} \lvert \mathcal F_{t-1}) &= \bm 1_{H_0} \bm 1_{A_{t-1}} \mathbb{P}(\{a_t \neq  i,X_t\in U_i\}\cap E_{t} \cap A_{t-1} \lvert \mathcal F_{t-1})\\
		&\leq \bm 1_{A_{t-1}}\bm 1_{H_0} \mathbb{P}(\triangle_{t} < 2L\CB \lambda_0 )\\
		&\leq \bm 1_{A_{t-1}} \bm 1_{H_0} \gamma(2L\CB \lambda_0)^\beta\\
		&\leq  \bm 1_{A_{t-1}} \bm 1_{H_0} \dfrac{p'}{2},
	\end{align*}
	where the last inequality is by the choice of $\lambda_0$.
	Then the proof is finished.
\end{proof}

\subsection{Proof of Lemma \ref{OLSTHM} }
We first establish the lower bound of the sample-covaraince matrix sampled by the greedy action based on the following matrix-martingale concentration result:

\begin{lemma}[Theorem~3.1 in \cite{tropp2011user}] \label{Tropp lemma}
	 Let $z^1,\dots,z^t$ be a sequence of random, positive-semidefinite $d\times d$ matrices adapted to a filtration $\mathcal F'_t$, let $Z_t\coloneqq \sum_{i = 1}^t z^{i}$ and $\tilde{Z}_t\coloneqq \sum_{i = 1}^t \E[z^{i}\lvert \mathcal F'_{i-1}]$. Suppose that $\lambda_{\max}(z^{i})\leq R^2$ almost surely for all $i$, then for any $\mu$ and $\alpha\in (0,1),$
	 \begin{displaymath}
			\mathbb P[\lambda_{\min}(Z_t)\leq (1-\alpha)\mu, \lambda_{\min}(\tilde{Z}_t)\geq \mu]\leq d(\dfrac{1}{e^\alpha(1-\alpha)^{1-\alpha}} )^{\mu/R^2}.
	 \end{displaymath}
\end{lemma}
Now we can show the following result:
\begin{lemma}\label{minimal eigenvalue with dependence} For $t_1<t_2 \in\mathbb N$ such that $(t_2-t_1) \cdot\dfrac{\kappa_l p'}{8K} >10\CB^2\log ({d}/{\alpha'})$, for a fixed $i\in [K]$ we have 
\begin{align*}
	\mathbb P( \lambda_{\min} (\sum_{t = t_1}^{t_2} X_tX_t\bm 1\{a_t = i\}) \leq  \dfrac{t_2-t_1}{8K}\kappa_l p',\sup_{t_1\leq t\leq t_2,i\in K_{opt}} \lVert\hat{\theta}_{t,i}-\theta_i\rVert\leq\lambda_0,H_0) \leq  \alpha'.
\end{align*} 
\end{lemma}

\begin{proof}
	Denote ${S}_{t_1,t_2}\coloneqq \cap_{t_1\leq t\leq t_2} A_t$, by Lemma~\ref{Eigenvalue condition} we have 
	\begin{align*}
	\lambda_{\min}( \sum_{t=t_1}^{t_2} \E[ X_tX_t^T\bm 1\{a_t = i\}\lvert \mathcal F_{t-1}] )&\geq  \sum_{t=t_1}^{t_2}\bm 1_{A_{t-1}}\bm 1_{H_0} \dfrac{\kappa_l p'}{2K}.
	\end{align*}
	That implies \begin{align*}
		&\mathbb{P}( \lambda_{\min} (\sum_{t = t_1}^{t_2} X_tX_t^T\bm 1\{a_t = i\}) \leq  \dfrac{t_1-t_2}{4K}\kappa_l p',S_{t_1,t_2},H_0)\\
		\leq & \mathbb{P}( \lambda_{\min} (\sum_{t = t_1}^{t_2} X_tX_t^T\bm 1\{a_t = i\}) \leq  \dfrac{t_1-t_2}{4K}\kappa_l p', \E[ X_tX_t^T\bm 1\{a_t = i\}\lvert \mathcal F_{t-1}] )\geq  (t_2-t_1) \dfrac{\kappa_l p'}{2K}).
	\end{align*}
Then selecting $\alpha = 1/2$ and $\mu \ = (t_2-t_1) \cdot\dfrac{\kappa_l p'}{4K}$ in Lemma~\ref{Tropp lemma}, we have 
\begin{align*}
	\mathbb{P}(\lambda_{\min}(\sum_{t = t_1}^{t_2} X_tX_t^T\bm 1\{a_t = i\}) \leq (t_2-t_1) \dfrac{\kappa_l p'}{8K},  S_{t_1,t_2},H_0 )   \leq d(\dfrac{1}{\sqrt{e/2}})^{10\log (\frac{d}{\alpha'})}\leq  \alpha'.
\end{align*}
That leads to the claim.
\end{proof}

In warm up stage,  we have the following lemma.\begin{lemma}\label{lemmaE8} As long as $s_0\geq C(\kappa_{l}p')^{-2}\max\{\log \dfrac{1}{\alpha} ,d\}$ for some absolute constant $C$, we have with probability at least $1-\alpha,$ \begin{align*}
		\lambda_{\min}(\sum_{t=1}^{Ks_0} \bm 1\{a_t = i\} X_tX_t^T)^{-1}\leq \dfrac{2}{s_0 p'\kappa_l},\quad \forall i\in [K].
	\end{align*}    
	\end{lemma}

\begin{proof}
	 Since $X_{t}$ are i.i.d. for $(i-1)s_0+1 \leq t\leq is_0$,  using classical concentration results for i.i.d. sub-gaussian covariance matrix result (e.g. Theorem~6.5 in \cite{wainwright2019high} ), we have when $s_0>C  (\kappa_{l}p')^{-2}  \max\{\log \dfrac{1}{\alpha} , d\}$, with probability at least $1-\alpha$, 
	 \begin{align*}	
	 	\lVert \dfrac{1}{s_0}\sum_{t=1}^{Ks_0}\bm 1\{a_t = i\}X_tX_t^T-\E[X_1X_1^T]\rVert &\leq c_1(\sqrt{\dfrac{d}{s_0}}+{\dfrac{d}{s_0}})+c_2\max\{\sqrt{\dfrac{\log 1/\alpha}{s_0}},\dfrac{\log1/\alpha}{s_0} \}\nonumber\\
	 	&\leq  c_3(\sqrt{\dfrac{d}{s_0}}+\sqrt{\dfrac{\log(1/\alpha)}{s_0}})\\
	 	&\leq p'\kappa_l/2.
	 \end{align*}
	On the other hand, we have by Markov's inequality  \begin{align*}
		\lambda_{\min}\E[X_1X_1^T] \geq  \sum_{i\in K_{opt}} \lambda_{\min} \E[X_1X_1^T\bm 1\{X_1\in U_i\} ]\geq  \kappa_l p' 
	\end{align*}. Thus we have with probability at least $1-\alpha$, 
	 \begin{align*}
	 	\lambda_{\min}(\sum_{t= 1}^{Ks_0} X_tX_t^T)\geq s_0 p'\kappa_l/2.
	 \end{align*}
\end{proof}	

Now we can claim our first result about the private OLS-estimator in the warm up stage:

\begin{lemma}\label{lemmaAwarmup1}
	 Selecting $s_0$ as in Lemma~\ref{lemmaE8}  . For the warm up stage with private-OLS-estimator and length $Ks_0$, we have for any $\alpha>0$, with probability at least $1-\alpha$, \begin{align*}
	\sup_{i\in [K]}	\lVert \hat{\theta}_{t,i}-{\theta}_{i} \rVert\leq \dfrac{(4\CB\sigma_{\epsilon} + \pfactorc )\sqrt{t(\log(\frac{TK}{\alpha})+d) }}{s_0p'\kappa_l} \quad\text{holds for all $Ks_0 \leq t\leq T$}.
	\end{align*}
\end{lemma}
\begin{proof}
Denote $U_t = \sum_{s=1}^{t}(\bm 1\{a_s = i\}X_sX_s^T +(\bm 1\{a_s = i,s\leq Ks_0\}+\bm 1\{s>Ks_0\} ) W_{s})+\tilde{c}\sqrt{t}I_d, $ 
	we have \begin{align*}
		\hat{\theta}_{t,i} &= U_t^{-1}(\sum_{s=1}^t\bm 1\{a_s = i\} X_sy_s+(\bm 1\{ a_s = i,s\leq Ks_0\}+\bm 1\{s>Ks_0\})\xi_s  )\\
		&=U_t^{-1}(\sum_{s=1}^t\bm 1\{a_s = i\} [X_sX_s^T\theta_i+X_s\epsilon_s]+(\bm 1\{ a_s = i,s\leq Ks_0\}+\bm 1\{s>Ks_0\})\xi_s  )\\
		& = \theta_i+U_t^{-1}
		(\sum_{s=1}^t(\bm 1\{a_s = i\}X_s\epsilon_s+(\bm 1\{ a_s = i,s\leq Ks_0\}+\bm 1\{s>Ks_0\})(\xi_s-W_s\theta_i))-\tilde{c} \sqrt{t} I_d\theta_i  ).
	\end{align*}
	By $\lVert \sum_{s=1}^{Ks_0}\bm 1\{a_s = i\}W_s+\sum_{s = Ks_0+1}^t W_s\rVert\leq \tilde{c}\sqrt{t}, \forall Ks_0\leq t\leq T,i\in [K] $ with probability at least $1-\alpha$, we have with probability at least $1-2\alpha$, \begin{align*}
		[\lambda_{\min}(U)]^{-1}\leq \lambda_{\min}(\sum_{s=1}^{Ks_0} \bm 1\{a_s = i\}X_sX_s^T)^{-1}\leq \dfrac{2}{s_0p'\kappa_l},\quad\forall  Ks_0\leq t\leq T.
	\end{align*}
	On the other hand, we have  by the concentration of sub-gaussian random vector, the following bounds hold with probability at least $1-\alpha/(T^2K)$:\begin{align}
	&\lVert	\sum_{s=1}^{t}\bm 1\{a_s = i\}X_s\epsilon_s\rVert\leq C\CB\sigma_\epsilon\sqrt{t(d+\log(TK/\alpha))}, \label{normbound1} \\
	&\lVert	\sum_{s=1}^{t}(\bm 1\{a_s = i,s\leq Ks_0\}+\bm 1\{s>Ks_0\})\xi_s\rVert\leq C\pfactorc \sqrt{t(d+\log(TK/\alpha))}, \label{normbound2} \\
	&\lVert	\sum_{s=1}^{Ks_0}\bm 1\{a_s = i\}W_s\theta_i+\sum_{s=Ks_0+1}^tW_s\theta_i\rVert\leq \tilde{c}\sqrt{t}\lVert \theta_i\rVert\leq C\pfactorc\sqrt{t(d+\log(TK/\alpha) )}. \label{normbound3}
	\end{align}
	Gathering all bounds together, we have with probability at least $1-(2+\dfrac{1}{T^2})\alpha$, \begin{align*}
		\sup_{i\in [K]}\lVert \hat{\theta}_{t,i}-\theta_i\rVert\le \dfrac{2C}{s_0p'\kappa_l}(\CB\sigma_\epsilon+\sigma_{\varepsilon,\delta}) \sqrt{t(\log({TK}/{\alpha})+d)}. 
	\end{align*}
	That finishes the proof.
\end{proof}

\begin{lemma}
\label{initlemmaE10}
As long as 
\begin{align*}
s_0&\geq CK(\dfrac{\CB\sigma_\epsilon+\pfactorc}{\min\{\lambda_0,h_0\}  p'\kappa_l  })^2(d+\log(TK/\alpha)),\end{align*} we have with probability at least $1-\alpha$,\begin{align}
    &\sup_{i\in [K]}\lVert \hat{\theta}_{Ks_0,i}-\theta_i\rVert_2\leq \min\{\lambda_0,h_0 \} ,\label{INIT1} \\
	&\sup_{i\in K_{opt}}\lVert \hat{\theta}_{s,i}-\theta_i\rVert_2\leq \lambda_0\text{ holds uniformly for $Ks_0\leq s\leq (K+1)s_0$,}\label{INIT2}\\
	& C\dfrac{K(\CB\sigma_{\epsilon}+\pfactorc)\sqrt{d+\log(TK/\alpha)} }{\sqrt{t-Ks_0}\kappa_l p'}\leq  \lambda_0 \text{ holds for all $t\geq 2Ks_0$ .}\label{INIT3}
\end{align} 
\end{lemma}
\begin{proof}
	To show \eqref{INIT1},\eqref{INIT2}, we can just plug the value of $s_0$ into the upper bound in Lemma~\ref{lemmaAwarmup1}. \eqref{INIT3} comes directly from the value of $s_0$.
\end{proof}
Now, we can show the following result: 
\begin{lemma}\label{lemmaAwarmup2}
	With the choice of $s_0$ same as in Lemma \ref{initlemmaE10}, for $t>Ks_0$, denote $t' = t - Ks_0  $ and $\tilde{t}_0 = 2Ks_0$, we have if 
	$$
	H_0 \text{   holds and   } \lVert \hat{\theta}_{t,i}-\theta_i\rVert_2\leq  \min\{\tilde{U}_s(\alpha),\lambda_0\} \text{ holds uniformly for $i\in K_{opt},\tilde{t}_0 \leq s\leq t $  },
	$$
	 with probability at least $1-\sum_{j=1}^{t'}\dfrac{2}{j^2}\alpha$, then  
	 $$
	 H_0 \text{  holds and   } \lVert \hat{\theta}_{t,i}-\theta_i\rVert_2\leq  \min\{\tilde{U}_s(\alpha),\lambda_0\} \text{ holds uniformly for $i\in K_{opt},  \tilde{t}_0 \leq s\leq t+1 $  },
	 $$
	 with probability at least $1-\sum_{j=1}^{t'+1}\dfrac{2}{j^2}\alpha$ , where \begin{align*}
		\tilde{U}_s(\alpha) = C\dfrac{K(\CB\sigma_{\epsilon}+\pfactorc)\sqrt{d+\log(TK/\alpha)} }{\sqrt{s}\kappa_l p'}.
	\end{align*}
\end{lemma}
\begin{proof}
	Denote ${S}_{\tilde{t}_0,t} =\{ \lVert \hat{\theta}_{s,i}-\theta_i\rVert\leq\min\{ \tilde{U}_s(\alpha) ,\lambda_0\},\forall K\in K_{opt}, \forall  \tilde{t}_0 \leq s\leq t \},\tilde{A}_{t}= \{\sup_{i\in K_{opt}}\lVert \hat{\theta}_{i,t}-\theta_i\rVert\leq \tilde{U}_t(\alpha) \} $ , we have by   Lemma~\ref{minimal eigenvalue with dependence} \begin{align*}
		\mathbb{P}(S_{\tilde{t}_0,t}, H_0, \lambda_{\min}( \sum_{s=1}^{t} X_sX_s\bm 1\{a_s = i\}) > \dfrac{t'\kappa_l p'}{8K})\geq 1-\dfrac{\alpha}{2KT^2}.
	\end{align*}
	Applying the inequalities \eqref{normbound1},\eqref{normbound2} \eqref{normbound3}, we have \begin{align*}
	\mathbb{P}(H_0,{S}_{\tilde{t}_0,t},A_{t+1})&\geq 1-\sum_{j=1}^{t'}\dfrac{2}{j^2}\alpha-\dfrac{3\alpha}{2T^2}- \sum_{i\in K_{opt}}\mathbb{P}(H_0,{S}_{\tilde{t}_0,t},\tilde{A}_{t+1},\lambda_{\min}( \sum_{s=1}^{t} X_sX_s\bm 1\{a_s = i\}) \leq  \dfrac{t'\kappa_l p'}{4})\\
		&\geq  1-\sum_{j=1}^{t'}\dfrac{2}{j^2}\alpha-\dfrac{2\alpha}{T^2}\\
		&\geq 1-2{\sum_{j=1}^{t'+1}}\dfrac{1}{j^2}\alpha.
	\end{align*}
	 By the selection of $s_0$ ,  we have $\tilde{U}_{s}(\alpha)\leq \lambda_0$ for $\tilde{t}_0 \leq s\leq t+1$,  and as a result, $
	 	\mathbb{P}(H_0,S_{\tilde{t}_0,t+1}) =	 \mathbb{P}(H_0,S_{\tilde{t}_0,t},\tilde{A}_{t+1}).$
	 Thus the claim holds.
\end{proof}
\begin{proof}[Proof of Lemma~\ref{OLSTHM}] Lemma~\ref{OLSTHM} is implied directly by Lemma~\ref{lemmaAwarmup2} and Lemma~\ref{initlemmaE10}. 
	
\end{proof}

\subsection{Proof of Lemma~\ref{SGDTHM} }
\begin{proof}
For the estimator $\hat\theta_{Ks_0,i} $ at the end of warm up stage, since the action is independent of the contexts, every $\hat\theta_{Ks_0,i}$ can be seen as an output of performing private gradient descent over $s_0$ i.i.d. samples. Without loss of generality, we perform the analysis for the parameter of the first arm $\hat\theta_{Ks_0,1}$ (notice that by the sampling strategy in the warm up stage, we have $\hat\theta_{Ks_0,1}=\hat\theta_{s_0,1}$). The result for other $\hat\theta_{Ks_0,i}$ can be established using the same argument. For $2 \leq t\leq s_0$, \begin{align*}
	\lVert \hat{\theta}_{t,i}-\theta_i\rVert^2 & =  \lVert \hat{\theta}_{t-1,i}-\eta_{t}\hat{g}_{t} -\theta_i\rVert^2\\
	&=\lVert \hat{\theta}_{t-1,i}-\theta_i\rVert^2-2\eta_{t}\hat{g}_{t}^T(\hat\theta_{t-1,i}-\theta_i)+2\eta_{t}^2\lVert \hat{g}_{t}\rVert^2
\end{align*}
Here $\hat{g}_{t}: = \Psi_{\varepsilon}[(\mu(X_t^T\hat\theta_{t,i})-r_t)X_t]$, by the unbiasedness of 
$\Psi_{\varepsilon}$ in Lemma~\ref{privacy-bounded-vector} we have 
\begin{align*}
	& \E[ \Psi_\varepsilon((\mu(X_t^T\hat\theta_{t-1,i})-r_t) X_t)^T(\hat\theta_{t-1,i}-\theta_i)  \lvert \mathcal F_{t-1} ]\\
	& =  \E[(\mu(X_t^T\hat\theta_{t-1,i})-\mu(X_t^T\theta_{i}))X_t^T(\hat\theta_{t-1,i}-\theta_i )\lvert \mathcal F_{t-1} ] \\
	& \geq \zeta \E[[X_t^T(\hat\theta_{t-1,i}-\theta_i)]^2 \lvert \mathcal F_{t-1} ]\\
	& \geq  \zeta \kappa_l   p' \lVert \hat\theta_{t-1,i}-\theta_i\rVert^2. 
\end{align*}
We get \begin{align*}
	\lVert \hat\theta_{t,i}-\theta_i\rVert^2 \leq (1-2\zeta  \kappa_l p' \eta_t )\lVert \hat\theta_{t-1,i}-\theta_i\rVert^2+2\eta_t(\E[\hat{g}_{t}\lvert \mathcal F_{t-1} ]- \hat{g}_t )^T(\hat\theta_{t-1,i}-\theta_i)+2\eta_t^2\lVert \hat{g}_t\rVert^2.
\end{align*}
Notice $\lVert \hat{g}_t\rVert_2^2$ is upper bounded by $r_{\varepsilon,d}^2$. Now using the same argument as in the proof of Proposition 1 of \cite{rakhlin2011making} leads to the following result: \begin{lemma}
	If we pick $\eta_t = 1/(\zeta\kappa_l p't)$ in the warm up stage, then with probability at least $1-\alpha$, \begin{align}\label{ggdd0}
		\sup_{i\in [K]} \lVert \hat{\theta}_{Ks_0,i}- {\theta}_{i}\rVert^2\leq C \dfrac{(\log(\log(KT)/\delta)+1)r_{\varepsilon,\delta}^2 }{ \zeta^2\kappa_l^2 p'^2s_0}.
	\end{align}
\end{lemma}

Notice that in our algorithm, when $t>Ks_0$, for any $i\in K_{opt}$, the private gradient descent formula is given by\begin{align*}
	 \hat\theta_{t,i} = \hat\theta_{t-1,i}-\eta_t \tilde{g}_{t},
\end{align*} 
with  $\tilde{g}_t = \bm 1\{a_t = i\} \hat{g}_{t}+\bm 1\{a_t \neq i \} \Psi_\varepsilon(0).$ Again without loss of generality we assume that $1\in K_{opt}$, and we provide the analysis for $i=1$, the argument is same for other $i\in K_{opt}$: \begin{align*}
	\E[\tilde{g}^T(\hat{\theta}_{t-1,1}-\theta_1)\lvert \mathcal F_{t-1} ]& = 
	\E[\bm 1\{a_t = i\} \hat{g}^T(\hat{\theta}_{t-1,1}-\theta_1)\lvert \mathcal F_{t-1} ]\\
	&=\E[ \bm 1\{a_t = i\} (\mu(X_t^T\hat\theta_{t-1,1})-\mu(X_t^T\theta_{i}))X_t^T(\hat\theta_{t-1,1}-\theta_1 )\lvert \mathcal F_{t-1} ]\\
	&\geq \zeta \E[ \bm 1\{a_t = i\}[X_t^T(\hat\theta_{t-1,1}-\theta_1 )]^2\lvert \mathcal F_{t-1} ]\\
	&\geq  \bm 1_{A_t, H_0} \zeta \kappa_l p'\eta_t \lVert \hat\theta_{t-1,1}-\theta_1\rVert^2/K
\end{align*}
select $\eta_t\coloneqq K_{opt}/(\zeta\kappa_lp't')$, with $t' = t-(K-1)s_0$ we have then  \begin{align*}
	\lVert \hat{\theta}_{t,1}-\theta_1\rVert_2^2\leq (1- \dfrac{2}{t'}\bm 1_{A_t, H_0})\lVert \hat{\theta}_{t-1,1}-\theta_1\rVert^2 + \dfrac{2K}{\zeta\kappa_l p' t'} (\E[\tilde{g}_t\lvert\mathcal F_{t-1} ]-\tilde{g}_t)^T(\hat{\theta}_{t-1,1}-\theta_1)+2(\dfrac{K r_{\varepsilon,d}}{\zeta\kappa_lp' t'})^2
\end{align*}
If we denote $S_t: =\cap_{s=Ks_0}^{t} A_s $, then using the above inequality recursively until $t = Ks_0+1$(i.e. until $t' = s_0+1$) , we have \begin{align*}
	\bm 1_{S_{t-1}, H_0} \lVert \hat{\theta}_{t,1}-\theta_1\rVert^2 &\leq  \dfrac{s_0(s_0-1)}{t'(t'-1)} \lVert \hat{\theta}_{Ks_0,1}-\theta_1\rVert^2	+2(\dfrac{K r_{\varepsilon,d}}{\zeta\kappa_lp' t'})^2\\
	&+\dfrac{2K}{(t'-1)t' \zeta\kappa_l p' }\sum_{s = Ks_0+1}^t (\E[\tilde{g}_s\lvert\mathcal F_{t-1} ]-\tilde{g}_s)^T(\hat{\theta}_{s-1,1}-\theta_1).
\end{align*}
Then it follows from the same proof as in Proposition 1 in \cite{rakhlin2011making} that for any fixed $Ks_0<t\le T$, we have with probability at least $1-\alpha/T$, \begin{align}\label{ggdd1}
	\bm 1_{S_{t-1}, H_0} \lVert \hat{\theta}_{t,1}-\theta_1\rVert^2 \leq \dfrac{s_0(s_0-1)}{t'(t'-1)} \lVert \hat{\theta}_{Ks_0,1}-\theta_1\rVert^2 +C \dfrac{K^2(\log(TK\log(TK)/\alpha)+1)r_{\varepsilon,d}^2}{\zeta^2\kappa_l^2p'^2t'},
\end{align}
Now choose $s_0\geq 2C \dfrac{K^2(\log(TK\log(TK)/\alpha)+1)r_{\varepsilon,d}^2}{ \zeta^2\kappa_l^2p'^2\min\{\lambda_0,h_0\}^2},$ so that the second term in \eqref{ggdd1} is less or equal to $\lambda_0/2$, 
we have $\mathbb{P}(S_{Ks_0+1},H_0)\geq 1-2\alpha$ by \eqref{ggdd0}. 
And by  calling \eqref{ggdd1} recursively we can get $\mathbb{P}(S_{t-1},H_0)> 1-2\alpha-\dfrac{t-Ks_0}{T}\alpha\geq 1-3\alpha,\forall Ks_0<t\leq T$. 
Then with probability at least $1-3\alpha$, we have \begin{align*}
	\lVert\hat{\theta}_{t,1}-\theta_1\rVert^2\leq C \dfrac{K^2(\log(3TK\log(TK)/\alpha))r_{\varepsilon,d}^2}{\zeta^2\kappa_l^2p'^2(t-(K-1)s_0)},\quad \forall Ks_0 < t\leq T.
\end{align*}
The above inequality is because the term $\dfrac{s_0(s_0-1)}{t'(t'-1)}\lVert \hat{\theta}_{Ks_0,1}-\theta_1\rVert^2\leq \dfrac{s_0(s_0-1)}{t'(t'-1)}\dfrac{\min\{\lambda_0,h_0\}}{2}$, which can be absorbed into the constant $C$. \end{proof}

\section{Proof of Theorem~\ref{thm:single-lower minimax bound}}\label{lower-bound-single}
In this section, we would give a proof on the Theorem~\ref{thm:single-lower minimax bound} by combining the argument in \cite{Han2020} and the divergence contraction inequality in \cite{duchi2018minimax}.

\begin{proof}[Proof of Theorem~\ref{thm:single-lower minimax bound}]
Consider the two-arm stochastic contextual bandit environment: for each d-dimensional context $i=1$ or $2$, $x_{t,i}\sim \mathcal N(0, \frac{1}{d}I_d)$ independently. 
If choosing action $a_t$ at time t, the reward $y_t$ is generated via $y_t = x_{t,a_t}^T\theta+\epsilon_t$ with $\epsilon_t\sim_{i.i.d.}\mathcal N(0,1)$. 
Given any fixed $\varepsilon$-LDP bandit algorithm $\pi$ with $\varepsilon\leq 1$, we denote its decision at  $t$-th step by $a_t$, by definition $a_t$ can be seen as a function of current contextual $x_{t,1},x_{t,2}$ and all history outputs $(x_{1, a_1},y_1,x_{2, a_2},y_2,\dots,x_{t-1, a_{t-1}},y_{t-1})$. 
Since the algorithm is under the $\varepsilon$-LDP constraint, each $a_t$ can only access $S_t\coloneqq(M_1( x_{1, a_1},y_1),M_2( x_{2, a_2},y_2),\dots,M_{t-1}( x_{t-1, a_{t-1}},y_{t-1}))$ with $M_1,\dots,M_{t-1}$ a sequence of $\varepsilon$-LDP mechanisms. 
We denote the distribution of $S_t$ by $Q_\theta^{t}$, and we have \begin{align}
	\begin{aligned}
	&\E_{\theta\sim Q_0}[\E_{Q_\theta}^t[(x_{t,a_t^*}-x_{t,a_t})^T\theta\lvert x_{t,1},x_{t,2}]] \\
	= &\E_{\theta\sim Q_0}[((x_{t,1}-x_{t,2})^T\theta)_+Q_{\theta}^t(a_t(S_t,x_t)=2) +((x_{t,2}-x_{t,1})^T\theta)_+Q_{\theta}^t(a_t(S_t,x_t)=1)],
	\label{lower-bound-est}
	\end{aligned}
\end{align}
where $(x)_+$ denote $\max\{x,0\}$ and $Q_0$ denote the uniform distribution over $\triangle S^{d-1}_1$ with $\triangle>0$ some positive number to be determined, we define $Q_1,Q_2$ as $$\dfrac{dQ_1}{dQ_0}: = \dfrac{ ((x_{t,1}-x_{t,2})^T\theta)_+}{Z_0},\quad \dfrac{dQ_2}{dQ_0}: = \dfrac{ ((x_{t,2}-x_{t,1})^T\theta)_+}{Z_0},$$
where $Z_0=\E_{Q_0}[((x_{t,1}-x_{t,2})^T\theta)_+] = \E_{Q_0}[((x_{t,2}-x_{t,1})^T\theta)_+]$ is the  normalization factor. Denote $r_t =\lVert x_{t,1}-x_{t,2}\rVert, u_t = r_{t}^{-1}(x_{t,1}-x_{t,2})$
, then the right hand side of \eqref{lower-bound-est} is lower bounded by  \begin{align*}
	&=Z_0(Q_1\circ Q_{\theta}^t(a_t(S_t,x_t) = 2)+Q_2\circ Q_{\theta}^t(a_t(S_t,x_t) = 1) ) \\
	&\geq_{(a)} Z_0 (1-\text{TV}( Q_1\circ Q_{\theta}^t,Q_2\circ Q_{\theta}^t ))\\
	&\geq_{(b)} \dfrac{Z_0}{2}\exp(-D_{KL}( Q_1\circ Q_{\theta}^t\lVert	 Q_2\circ Q_{\theta}^t ))\\
	&=_{(c)}\dfrac{Z_0}{2}\exp(-D_{KL}( Q_1\circ Q_{\theta}^t\lVert	 Q_1\circ Q_{\theta-2(u_t^T\theta)u_t }^t ))\\
	&\geq_{(d)} \dfrac{Z_0}{2}\exp(-\E_{Q_1}[ D_{KL}(  Q_{\theta}^t\lVert	 Q_{\theta-2(u_t^T\theta)u_t }^t   )]), \tag{F.1}\label{tagF1}
\end{align*}
where $D_{KL}(\cdot \lVert \cdot)$ denote the KL-divergence, $TV(\cdot,\cdot)$ denote the total variation distance and $Q_i\circ Q_{\theta}^t$ means $\E_{\theta\sim Q_i}[Q_{\theta}^t ]$. The (a) inequality comes from the fundamental limit of two-point testing (see e.g. Section 15.2 in \cite{wainwright2019high}), and the (b) inequality comes from Lemma 2.6 of \cite{tsybakov2008introduction}, the (c) equality comes from Lemma 8 in \cite{Han2020}  and the (d) inequality comes from the strongly-convexity of KL-divergence. 
Now  by chain rule of KL-divergence, the divergence contraction inequality in Theorem 1 of \cite{duchi2018minimax} and the formula of KL-divergence between Gaussian distributions, we have \begin{align*}
	D_{KL}(  Q_{\theta}^t\lVert	 Q_{\theta-2(u_t^T\theta)u_t }^t)  &= \sum_{s = 1}^{t-1} \E_{Q_{\theta}^{s-1}}[D_{KL}( P^t_{\theta}(\cdot \lvert S_{s-1} ) \lVert  P^t_{\theta-2(u_t^T\theta)u_t}(\cdot \lvert S_{s-1} ) )  ]\\
	&\leq \sum_{s = 1}^{t-1} \dfrac{c}{2}(e^\varepsilon-1)^2(2(u_t^T\theta)^2\lVert u_t\rVert^2) 
\end{align*}
By the argument of  in \cite{Han2020}, we have \eqref{tagF1} is lower bounded by \begin{align*}
	\dfrac{r_t\triangle}{C\sqrt{d}} \exp(-C\dfrac{(e^\varepsilon-1)^2\triangle^2 }{d+1}u_t^T(\sum_{s=1}^{t-1}x_{t,a_t}x_{t,a_t}^T) u_t).
\end{align*}
Now taking expectation over $x_{t,1},x_{t,2}$, and using the convexity of function $f(x) = \exp(-x)$ we get \begin{align*}
	\E_{x}\E_{\theta}\E_{Q_{\theta}^t}[x_{t,a_t^*}-x_{t,a_t}^T] \geq \dfrac{\triangle}{C\sqrt{d}}\exp(-\dfrac{C(e^\varepsilon-1)^2\triangle^2 t}{d^2}).
\end{align*}
Selecting $\triangle \asymp \dfrac{d }{(e^\varepsilon-1)\sqrt{t}} $ and taking summation over $1\leq t\leq T$ leads to $\Omega(\sqrt{Td}/(e^\varepsilon-1))$ lower bound, finally noticing $e^\varepsilon-1 \leq C \varepsilon$ for $\varepsilon\leq 1$ leads to the desired lower bound when $\varepsilon\leq 1.$
\end{proof}

\section{Auto Loan Experiment Details}
\label{Real-Experiment}
We use On-Line Auto Lending dataset CRPM-12-001 in our real data case study\footnote{On-Line Auto Lending dataset CRPM-12-001 provided by Columbia University \url{https://www8.gsb.columbia.edu/cprm/research/datasets}.}. We use the same features selection as in \cite{ban2020personalized, cheung2018hedging} in the dataset and select FICO score, the term of contract, the loan amount approved, prime rate, the type of car, and the competitor’s rate as the feature vector for each customer. Note that a description of the data set (with descriptive statistics on the demand and available features) is available in \cite{ban2020personalized}. 
The objective is to offer a personalized lending price (from a range of choices) based on personal information such as FICO score to a customer who will either accept or reject it. 
In contrast to linear bandits, the binary reward is non-linear. Therefore we leave LDP-UCB and LDP-OLS out of considerations.
To formulate a bandit environment, first we need to recover the underlying true parameter. Since the lender’s decision, i.e., the price for each customer, is not presented in the dataset, we follow \cite{ban2020personalized, cheung2018hedging} and impute it by using the net-present value of futher payment minus the loan amount, i.e., 
$$
p=\text { Monthly Payment } \times \sum_{\tau=1}^{\text {Term }}(1+\text { Rate })^{-\tau}-\text { Loan Amount }.
$$
After imputing the loan prices, to represent customers' binary loan choices, we employ the logit demand model. 
To be specific, given a price $p$ and a context $x\in \R^d$, the binary variable \emph{apply} takes value of 1 with probability $\frac{\exp(v)}{1+\exp(v)}$ and takes value of 0 with probability $\frac{1}{1+\exp(v)}$ where the linear predictor $v = (x, px)^T\theta^{\star}$. 
We conduct one-hot encoding for categorical features in the dataset and use the python package {\tt sklearn} \cite{scikit-learn} for the estimation of the underlying parameter $\theta^{\star}$. 
We use the interval $[0, 25000]$ as the feasible region of the prices, which covers the lending prices computed from the dataset, and we discrete the feasible region uniformly into 25 options $\{p_{i}\}_{i\in[25]}$.  
We use LDP-SGD and LDP-GLOC to sequentially compute the loan prices for the $10^5$ with randomly selected customers in the dataset, and compute the company’s expected regret based on the population model mentioned above.

\begin{figure}[H]
	\centering
	\begin{tikzpicture}[scale = 1.2]
	  \definecolor{LDP-OLS}{rgb}{0.12156862745098,0.466666666666667,0.705882352941177}
	  \definecolor{LDP-SGD}{rgb}{1,0.498039215686275,0.0549019607843137}
	  \definecolor{LDP-UCB}{rgb}{0.172549019607843,0.627450980392157,0.172549019607843}
	  \definecolor{LDP-GLOC}{rgb}{0.83921568627451,0.152941176470588,0.156862745098039}
	  
	\begin{groupplot}[group style={group size=1 by 1}]

	\input{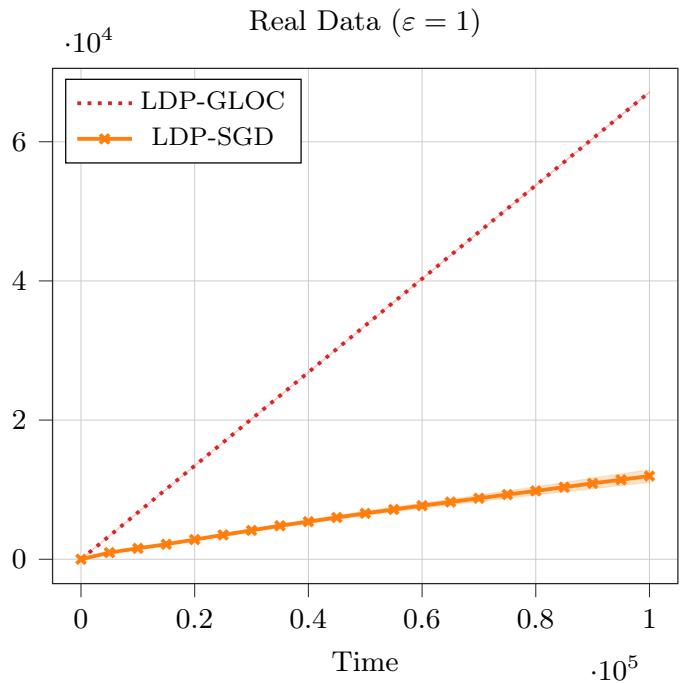}
	
	\end{groupplot}
	\end{tikzpicture}
	\caption{
	  We perform 10 replications for each case and plot the mean and 0.5 standard deviation of their regrets.
	  }
	\label{fig:numerical}
\end{figure}

\end{document}